%% file: main.tex
\RequirePackage{fix-cm}
\documentclass[twocolumn]{svjour3}          
\smartqed  
\usepackage{xcolor}
\usepackage{tabularx}
\usepackage{multirow}
\usepackage{graphicx}
\usepackage{makecell}
\usepackage{amsmath}
\usepackage{color}
\usepackage{float}
\usepackage{rotating}
\usepackage{placeins}
\usepackage{amsfonts}
\usepackage{bm}

\def\eg{\textit{e.g.}}
\def\ie{\textit{i.e.}}

\usepackage[round,authoryear,comma,sort&compress,numbers]{natbib} 
\usepackage[colorlinks=true,linkcolor=blue,citecolor=black,urlcolor=blue]{hyperref}

\usepackage{diagbox}
\usepackage{cancel} 
\usepackage{mathrsfs} 
\usepackage{eqnarray} 

\usepackage{ctable} 

\usepackage{algpseudocode}%
\usepackage{listings}%
\usepackage{bbding}
\usepackage{bbm}
\usepackage[ruled]{algorithm2e}
\usepackage[caption=false, font=footnotesize]{subfig}
\usepackage{fontawesome}
\usepackage{array}
\newcommand{\PreserveBackslash}[1]{\let\temp=\\#1\let\\=\temp}
\newcolumntype{C}[1]{>{\PreserveBackslash\centering}p{#1}}
\newcolumntype{R}[1]{>{\PreserveBackslash\raggedleft}p{#1}}
\newcolumntype{L}[1]{>{\PreserveBackslash\raggedright}p{#1}}

\newcommand{\extension}[1]{\textcolor{black}{#1}}
\newcommand{\revision}[1]{\textcolor{black}{#1}}
\newcommand{\rerevision}[1]{\textcolor{black}{#1}}
\newcommand{\major}[1]{\textcolor{black}{#1}}

\definecolor{battleshipgrey}{rgb}{0.52, 0.52, 0.51}
\definecolor{capri}{rgb}{0.0, 0.75, 1.0}
\definecolor{mediumspringgreen}{rgb}{0.0, 0.98, 0.6}

\journalname{International Journal of Computer Vision}
\begin{document}

\title{Softmax-free Linear Transformers
}


\author{
	Jiachen Lu$^1$, \and
	Junge Zhang$^1$, \and 
	Xiatian Zhu$^2$, \and
        Jiafeng Feng$^1$, \and
        Tao Xiang$^2$, \and
        Li Zhang$^1$\textsuperscript{\faEnvelopeO}
}




\institute{
	Corresponding author: Li Zhang  \at
             \email{lizhangfd@fudan.edu.cn}          \\
$^1$ School of Data Science, Fudan University, Shanghai, China \\
$^2$ University of Surrey, Guildford, UK
}

\date{15 March 2024}

\maketitle

\input{sections/0_abstract.tex}
\input{sections/1_introduction.tex}
\input{sections/2_related_work.tex}
\input{sections/3_method.tex}
\input{sections/4_experiments.tex}
\input{sections/5_conclusion.tex}
\input{sections/7-appendix_short}


\clearpage
\bibliographystyle{spbasic}      
\bibliography{main}   


\end{document}

%% file: sections/0_abstract.tex
\begin{abstract}
Vision transformers (ViTs) have pushed the state-of-the-art for visual perception tasks.
The self-attention mechanism underpinning the strength of ViTs has a quadratic complexity in both computation and memory usage. 
This motivates the development of approximating the self-attention at linear complexity.
However, an in-depth analysis in this work reveals that existing methods are either theoretically flawed or empirically ineffective for visual recognition. 
We identify that their limitations are rooted in the inheritance of {\em softmax} based self-attention during approximations, that is, normalizing the scaled dot-product between token feature vectors using the softmax function.
As preserving the softmax operation challenges any subsequent linearization efforts. 
By this insight, a family of {\em SOftmax-Free Transformers} ({\bf SOFT}) are proposed.
Specifically,
a Gaussian kernel function is adopted to replace the dot-product similarity, enabling a full self-attention matrix to be approximated under low-rank matrix decomposition. 
For computational robustness, we estimate the Moore-Penrose inverse using an iterative Newton-Raphson method in the forward process only, while calculating its theoretical gradients only once in the backward process.
To further expand applicability (\eg, dense prediction tasks), an efficient symmetric normalization technique is introduced.
Extensive experiments on ImageNet, COCO and ADE20K show that our SOFT significantly improves the computational efficiency of existing ViT variants.
With linear complexity, much longer token sequences are permitted by SOFT, resulting in superior trade-off between accuracy and complexity.
Code and models are available at \url{https://github.com/fudan-zvg/SOFT}.
\keywords{Transformer \and linear complexity \and softmax normalization \and softmax-free \and Gaussian attention}
\end{abstract}

%% file: sections/1_introduction.tex
\section{Introduction}\label{intro}
\input{figure/para_top1_memory}

Recently the step change brought by Transformers \citep{vaswani2017attention} in natural language processing (NLP) \citep{devlin2018bert,brown2020language} seems to have arrived in vision \citep{dosovitskiy2020image,yuan2021tokens,zhu2020deformable,zheng2021rethinking}.  Indeed, with less 
 inductive bias in its architecture design than Convolution neural networks (CNNs), pure Vision Transformer (ViT) \citep{dosovitskiy2020image} and its variants have shown to be able to outperform CNNs on various vision tasks \citep{d2021convit,jaegle2021perceiver}.
However, there is a bottleneck in any Transformer based model, namely its quadratic complexity in both computation and memory usage. This is intrinsic to the self-attention mechanism:
given a sequence of tokens (\eg, words or image patches) as input, the self-attention module iteratively learns the feature representations
by relating one token to all other tokens. This results in a quadratic complexity $O(n^2)$ with the token sequence length $n$
in both computation (time) and memory (space) since an $n\times n$ sized attention matrix needs to be computed and saved during inference. This problem is particularly acute in vision: a 2D image after tokenization will produce a far longer sequence than those in NLP even with a moderate spatial resolution. This quadratic complexity thus prevents a ViT model from modeling images at high spatial resolutions, which are often crucial for visual recognition tasks. 

A natural solution is to reduce  the complexity of self-attention computation via approximation. Indeed, 
there have been a number of attempts in NLP \citep{wang2020linformer,choromanski2020rethinking,kitaev2020reformer,xiong2021nystr}. For example, \citep{wang2020linformer} takes a naive approach by shortening the length of Key and Value via learnable projections. Such a coarse approximation would inevitably cause performance degradation. 
In contrast, \citep{choromanski2020rethinking,katharopoulos2020transformers} both leverage the kernel mechanism to approximate softmax normalization to linearize the computation in self-attention. 
\citep{kitaev2020reformer} instead adopts a hashing strategy to selectively compute the most similar pairs.
Recently, \citep{xiong2021nystr} uses Nystr{\"o}m matrix decomposition to reconstruct the full attention matrix
with polynomial iteration for approximating the pseudo-inverse of the landmark matrix. %
Nonetheless, softmax normalization is
simply duplicated across the matrix decomposition process,
which is  theoretically unsound. We empirically found  that none of these methods are effective when applied to vision (see Sec. \ref{sec:competitiors}). 

In this work, we identify that the  limitations of existing efficient Transformers are caused by the use of {\em softmax self-attention}, and for the first time propose a softmax-free Transformer. More specifically, in all existing Transformers (with or without linearization), 
a softmax normalization is needed on top of scaled dot-product  between token feature vectors \citep{vaswani2017attention}. Keeping this softmax operation challenges any subsequent linearization efforts.
To overcome this obstacle,  we introduce a novel 
{\em softmax-free self-attention} mechanism, named as SOFT,
with linear complexity $O(n)$ in both space and time.
Specifically, SOFT uses Gaussian kernel to define 
the similarity (self-attention) function without the need
for subsequent softmax normalization. With this softmax-free attention matrix, we further introduce a novel low-rank matrix decomposition algorithm for approximation. The robustness of the approximation is theoretically guaranteed by employing a Newton-Raphson method for reliably
computing the Moore-Penrose inverse of the matrix.

We make the following {\bf contributions}.
{\bf (I)} We introduce a novel {\em softmax-free Transformer} with linear space and time complexity.
{\bf (II)} Our attention matrix approximation is achieved  through a novel matrix decomposition algorithm with theoretical guarantee.
{\bf (III)} To evaluate our method for visual recognition tasks,
we design a family of generic backbone architectures
with varying capacities using SOFT as the core self-attention component.
Extensive experiments show that with a linear complexity (Figure \ref{fig:formercomparison}), our SOFT models  can take in as input much longer image token sequences. As a result, with the same model size, our SOFT outperforms the state-of-the-art CNNs and ViT variants on ImageNet \citep{deng2009imagenet} classification in the accuracy/complexity trade-off (Figure \ref{fig:paramettop1}).

\rerevision{
A preliminary version of this work was presented in NeurIPS 2021 spotlight \citep{lu2021soft}.
In this paper, we have further extended our conference version as follows:
(i) We improve the efficiency and robustness in computing Moore-Penrose inverse by using an iterative method in the forward process only while calculating its theoretical gradient only once in the backward propagation.
(ii) We analyze the limitations of our preliminary SOFT from a matrix spectral norm perspective, 
revealing the importance of normalization for enhancing the model's task generalizability.
(iii) We prove that the preliminary SOFT experiences a second-order increase in matrix spectral norm relative to the matrix size, making it fail in dense vision problems.
(iv) To address these limitations, we propose a normalized softmax-free self-attention,
\major{keeping linear complexity while enhancing performance, supported by both theoretical proof and extensive experiments}.
(v) The improved SOFT outperforms the state-of-the-art CNNs and ViTs for classification on ImageNet \citep{deng2009imagenet}, object detection on COCO \citep{lin2014microsoft}  and semantic segmentation on ADE20K \citep{zhou2019semantic}.
}

%% file: figure/para_top1_memory.tex
\begin{figure*}[htp]
    \centering
    \subfloat[\label{fig:paramettop1}]{
         \includegraphics[width=0.45\textwidth]{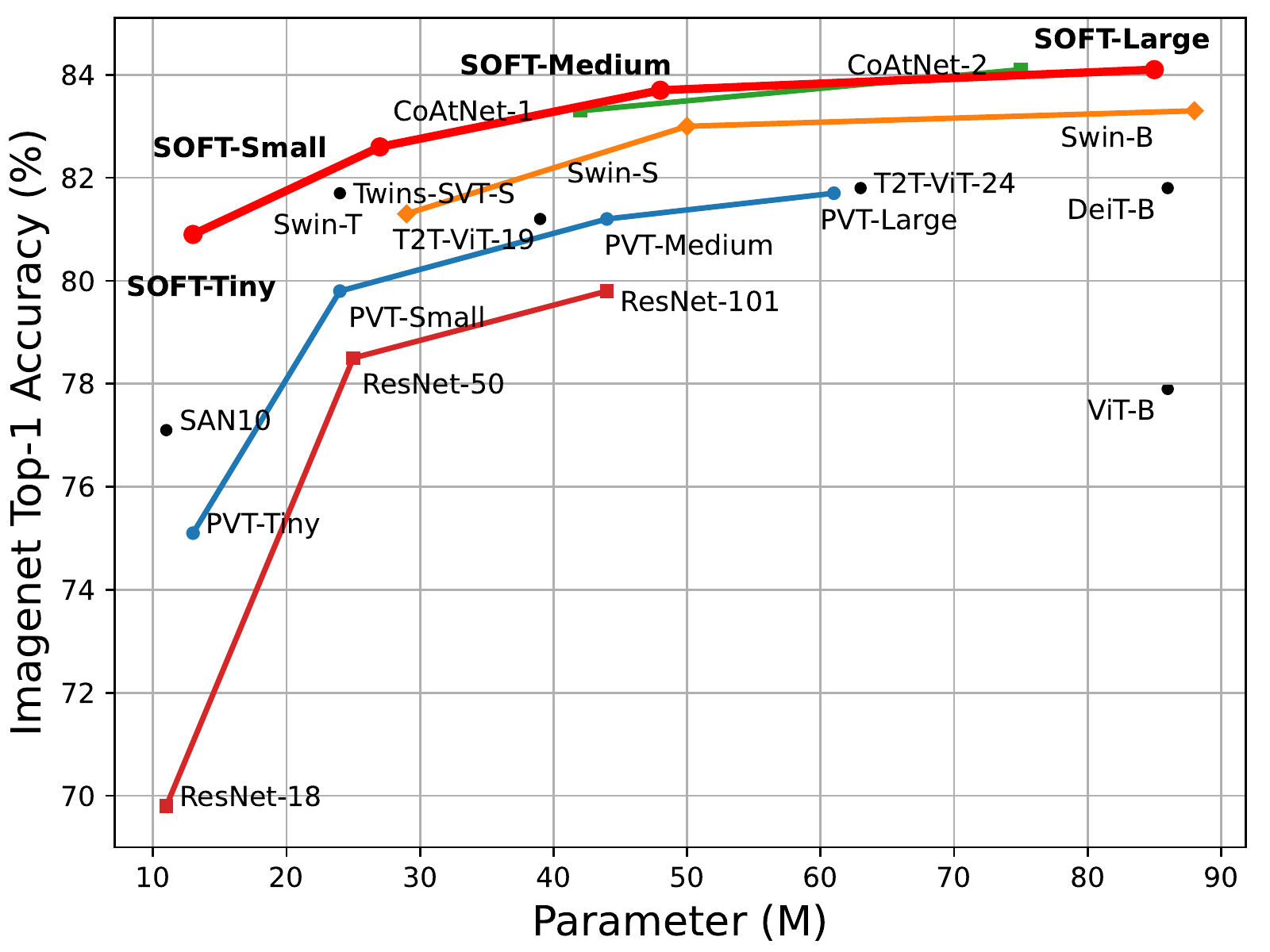}
    }
    \hspace{1.5em}
    \subfloat[\label{fig:formercomparison}]{
        \includegraphics[width=0.45\textwidth]{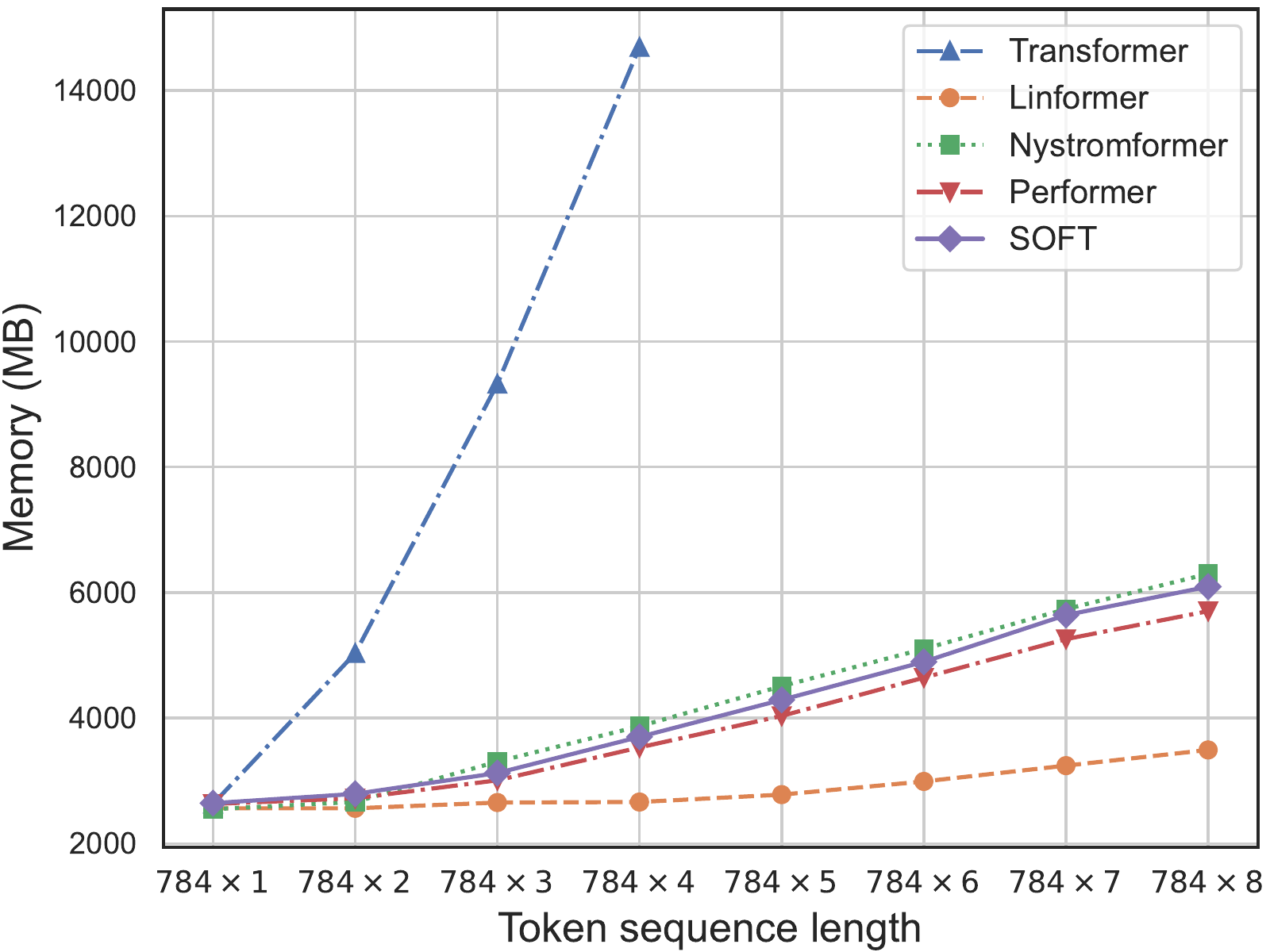}
    }
    \caption{Top-1 classification accuracy on ImageNet~\citep{deng2009imagenet} validation set with respect to parameters and the memory usage corresponding to the token sequence length in practice compared to other methods. 
    (a) Comparison with CNN models:
    ResNet~\citep{he2016deep} and 
    CoAtNet~\citep{dai2021coatnet}
    Transformer models: 
    PVT~\citep{wang2021pyramid}, 
    Swin~\citep{liu2021swin},
    DeiT~\citep{touvron2021training},
    ViT~\citep{dosovitskiy2020image}, 
    T2T-ViT~\citep{yuan2021tokens}, 
    Twins-SVT~\citep{chu2021twins} and
    SAN10~\citep{zhao2020exploring};
    (b) Comparison with Transformer~\citep{vaswani2017attention},
    Linformer~\citep{wang2020linformer},
    Nystr{\"o}former~\citep{xiong2021nystr} 
    and Performer~\citep{choromanski2020rethinking}. 
    The memory usage is measured with a batch size of 1 on a 16GB Tesla V100.}
\end{figure*}

%% file: sections/2_related_work.tex
\section{Related work}

\subsection{Vision Transformers}
There is a surge of research interests recently in
exploiting Transformers for visual recognition tasks
\citep{wang2018non,wang2021pyramid, guo2022beyond,yuan2021tokens,touvron2021training,zhang2020dynamic},
inspired by their remarkable success in NLP
\citep{vaswani2017attention,devlin2018bert,brown2020language}.
Core to these NLP and vision transformers is the same self-attention mechanism \citep{vaswani2017attention}
that computes a self-attention matrix by exhaustively comparing token pairs.
This means a quadratic complexity with the sequence length in both space and time, which thus limits the scalability of Transformers in dealing with long sequences. This limitation is more serious in vision than NLP:
To process an image with at least thousands of pixels,
patch-wise tokenization is a must for Transformers to control the computational cost. 
Given higher resolution images, the patch size also needs to be enlarged proportionally sacrificing the spatial resolution.
This limits the capability of Transformers, \eg, learning fine-grained feature representation as required in many visual recognition tasks.

\subsection{Linear Transformers} 
Recently, there have been a number of linear/efficient variants \citep{choromanski2020rethinking,wang2020linformer,katharopoulos2020transformers,kitaev2020reformer,tay2023efficient, peng2021random, kasai2021finetuning} of Transformers in NLP.
For example, \citep{wang2020linformer} learns to shrink the length of Key and Value based on a low-rank assumption.
\citep{kitaev2020reformer} adopts a hashing strategy to selective the most similar pairs
and only compute attention among them.
\citep{choromanski2020rethinking,katharopoulos2020transformers} utilize different kernel functions for approximating softmax-based self-attention matrix. 
\citep{peng2021random} applies random feature mapping on the sequences to approach the original softmax function. \citep{kasai2021finetuning} decreases the time and memory consumption of the attention matrix by replacing the softmax function with its linear-complexity recurrent alternative. 
When applied to visual recognition tasks, however, 
we show that these models
have considerable performance degradation
compared to the standard Transformers \citep{vaswani2017attention} (see Sec. \ref{sec:competitiors}).

The most related work to SOFT is \citep{xiong2021nystr} which
uses the Nystr{\"o}m matrix decomposition to avoid computing the full attention matrix.
However, this method suffers from several theoretical defects: (1)
As the standard self-attention needs to apply row-wise softmax normalization on the full attention matrix, a direct application 
of matrix decomposition is infeasible.
As a workaround without solid theoretical support, softmax is simply applied 
to all the ingredient matrices in \citep{xiong2021nystr}. Such an approximation is not guaranteed theoretically.
(2) With a polynomial iteration
method, 
it is not guaranteed that the generalized attention matrix inverse
can be computed when the matrix
is a nearly singular one in practice. 
In contrast to all the above methods, in this paper we propose a {\em softmax-free} self-attention mechanism that facilitates matrix decomposition for 
complexity minimization with theoretical guarantees.

%% file: sections/3_method.tex
\section{Method}
\label{method}
\subsection{Softmax-free self-attention formulation \label{sec:formulation}}

A schematic illustration of our model is given in Figure \ref{fig:Structure}. Let's first look at our attention module design.  Given a sequence of $n$ tokens $X\in\mathbb{R}^{n\times d}$ with each token represented by a  $d$-dimensional feature vector, self-attention \citep{vaswani2017attention} aims to discover the correlations of all token pairs exhaustively.

Formally, $X$ is first linearly projected into three $d_e$-dimensional spaces  (\texttt{query}, \texttt{key}, and \texttt{values}) as: 
{\small
\begin{equation} \label{eq:linear}
    Q=XW_q \in\mathbb{R}^{n\times d_e}, 
    K = X W_k \in\mathbb{R}^{n\times d_e}, 
    V = X W_v \in\mathbb{R}^{n\times d_e},
\end{equation}}
where $W_q, W_k, W_v \in \mathbb{R}^{d\times d_e}$
are learnable matrices.
\major{For equation simplicity, we omit the multi-head notation in self-attention operations.
However, it should be noted that multi-head mechanisms are employed throughout.}
Self-attention can be expressed in a generic formulation as:
\begin{equation}
    y_{i,:} = \sum_{j=1}^n \alpha(Q_{i,:}, K_{j,:}) \odot V_{j,:},
\end{equation}
where $\odot$ is the Hadamard product, and $i,j \in \{1,\cdots,n\}$ index the tokens.
The key self-attention function $\alpha:\mathbb{R}^{d_e} \times \mathbb{R}^{d_e} \rightarrow \mathbb{R}$ is composed of a nonlinear function $\beta:\mathbb{R} \rightarrow \mathbb{R}$ and a relation function $\gamma:\mathbb{R}^{d_e} \times \mathbb{R}^{d_e} \rightarrow \mathbb{R}$.
A dominant instantiation of $\alpha$
is the scaled dot-product based softmax self-attention \citep{vaswani2017attention}, defined as
\begin{equation}\label{eq:softmax-attn}
\beta(\cdot) = \text{softmax}(\cdot), \quad \gamma(Q_{i,:}, K_{j,:}) = \frac{1}{\sqrt{d_e}}\cdot Q_{i,:}^\top K_{j,:}.
\end{equation}
Whilst this softmax self-attention
has been the {\em de facto} choice and seldomly questioned, as discussed earlier it is not necessarily suited for linearization.
To facilitate the design of linear self-attention, we introduce a softmax-free self-attention function with the dot-product replaced by a Gaussian kernel as:
\begin{equation}\label{eq:softmax-free-attn}
\beta'(\cdot) = \text{exp}(\cdot), \quad
\gamma'(Q_{i,:}, K_{j,:}) = -\frac{1}{2\sqrt{d_e}}\cdot \| Q_{i,:} - K_{j,:}\|_2^2.
\end{equation}
\major{A dissection of Transformers by Tsai et al. \citep{tsai2019transformer} reveals negligible performance differences between asymmetric and symmetric kernels. 
To maintain the symmetric properties of the attention matrix as defined in Eq. \eqref{eq:softmax-attn}, we opt for identical projection matrices \( W_q \) and \( W_k \) in Eq. \eqref{eq:linear}, effectively setting \( Q = K \).
To further investigate the impact of symmetric kernels, additional experiments are conducted in our ablation studies of Section \ref{sec: abl}.}
Our self-attention matrix is then written as:
\begin{equation} \label{eq:soft_attn_func}
    S_{i,j} = \text{exp}\left(-\frac{1}{2\sqrt{d_e}}\cdot \| Q_{i,:} - K_{j,:}\|_2^2\right).
\end{equation}
For notation simplicity, we define the matrix formulation as:
$S = \text{exp}\left(Q\ominus K \right)$.

\noindent{\bf Remarks:}
Our self-attention matrix $S$ has three important properties: 
(1) It is symmetric;
(2) All the elements lie in a unit range of $[0,1]$; 
(3) All diagonal elements hold the largest value $1$ (self-reinforced), with
the bottom ones (corresponding to most dissimilar token pairs)
being close to $0$.
As Gaussian kernel is a positive definite kernel \citep{fasshauer2011positive},
$S$ is deemed a Gram matrix.
However, we find that when using our kernel-based self-attention matrix $S$ without linearization, 
the training of a transformer fails to converge.
\extension{
More discussion can be found in Section \ref{sec:normalization}.}

\input{figure/structure}

\subsection{Low-rank regularization via matrix decomposition with linear complexity}

To solve the convergence and quadratic complexity problems,
we leverage matrix decomposition as a unified solution 
with low-rank regularization.
In particular, we consider Nystr{\"o}m \citep{NIPS2000_nyst},
which is originally a low-rank matrix approximation algorithm.
This enables our model's complexity to be reduced significantly without computing the full self-attention matrix $S$.

We make this choice because our $S$ is positive semi-definite (\ie, a Gram matrix) without follow-up normalization which are all necessary conditions for Nystr{\"o}m.
In contrast, \citep{xiong2021nystr} totally ignores these requirements, leading to theoretical flaw in its approximation.

To define the  Nystr{\"o}m method formally,  let us express $S=\text{exp}\left(Q\ominus K \right)$ as a block matrix:
\begin{equation}
   S=\left[\begin{array}{cc}
    A & B \\
    B^\top & C \end{array}
    \right] \in \mathbb{R}^{n \times n},
\end{equation}
where $A\in \mathbb{R}^{m \times m}$, $B\in \mathbb{R}^{m \times (n-m)}$, $C \in \mathbb{R}^{(n-m) \times (n-m)}$ with $m \ll n$.
Through Nystr{\"o}m decomposition (see derivative details in 
Appendix \ref{sec:nystrom}),
an approximation can be represented as:
\begin{equation}
   \hat{S}=\begin{bmatrix}
    A \\
    B^\top
    \end{bmatrix}
        A^\dagger
    \begin{bmatrix}
        A & B
    \end{bmatrix} 
    = P^\top A^{\dagger} P, \quad \text{where} \quad
    P = 
    \begin{bmatrix}
    A & B
    \end{bmatrix},
\end{equation}
and $A^\dagger$ is the Moore-Penrose (a generalized) inverse of $A$.

\input{algorithm/soft}
\input{algorithm/newton_raphson}
\noindent{\bf Sampling: }
In the standard Nystr{\"o}m formulation, $A$ and $B$ are sub-matrices of $S$ obtained by randomly sampled $m$ tokens,
denoted as $\widetilde{Q}$. 
We call the sampled $\widetilde{Q}$ as {\em bottleneck tokens}.
However, we find empirically that random sampling is considerably sensitive to the choice of $m$.
We hence explore two additional options to leverage the structural prior of visual data:
(1) Using one convolutional layer with kernel size $k$ and stride $k$
to learn $\widetilde{Q}$,
and 
(2) Using average pooling with kernel size $k$ and stride $k$
to generate $\widetilde{Q}$.
For both, we need to reshape $Q$ to the form of $\mathbb{R}^{H \times W \times d_e}$.
Each slide of convolution or pooling produces a token.
We set $k$ according to the length of $Q$
such that $m$ tokens can be obtained.
Our experiments show that a convolution layer performs better in accuracy. 
We therefore use a convolution layer 
by default.

As $K$ is identical to $Q$, we have $\widetilde{K} = \widetilde{Q}$.
Given these $m$ tokens,
we then compute \textit{bottleneck attention} $A$ and $P$ as:
\begin{equation}
    A=\text{exp}(\widetilde{Q}\ominus \widetilde{K}), \quad
    P=\text{exp}(\widetilde{Q}\ominus K).
\end{equation}
We finally obtain a regularized self-attention matrix $\hat{S}$ of SOFT as:
\begin{equation} \label{eq:reg_attn}
    \hat{S} = \text{exp}\left(Q\ominus \widetilde{K}\right) \left(\text{exp}\left(\widetilde{Q}\ominus \widetilde{K}\right)\right)^{\dagger}
    \text{exp}\left(\widetilde{Q}\ominus K\right).
\end{equation}
The overall SOFT is summarized in Algorithm \ref{soft}.
The {\em low-rank} regularization is conducted as follows.
For computing the attention score
between any two tokens, 
we first correlate each of them with sampled tokens
using our self-attention function (Equation (\ref{eq:soft_attn_func})); With this correlation representation we then compute their similarity under the modulation of the generalized inverse of $\widetilde{Q}$'s correlation matrix.
Similar as the standard Nystr{\"o}m,
our design associates the input tokens w.r.t. a small space spanned by a set of sampled tokens, giving a proper estimation of the original attention relationships subject to a low-rank constraint. This method is proved in 
Appendix \ref{sec:nystrom}.

\noindent{\bf Moore-Penrose inverse:}
An accurate and commonly used way to calculate the Moore-Penrose inverse is to use Singular Value Decomposition (SVD). 
Given $A\in \mathbb{R}^{m\times m}$ and its SVD form $A = U\Sigma V^\top$ where $U,V$ are $m\times m$ unitary matrices and $\Sigma$ is a $m\times m$ diagonal matrix, the Moore-Penrose inverse of $A$ is $A^\dagger = V\Sigma^\dagger U^\top$.
Nevertheless, SVD is not friendly to the training process on GPU hence
harming the model training efficiency.
To solve this issue, we adopt 
the Newton–Raphson method.
It is an iterative algorithm with the $(k+1)$-th iteration formulated given the previous iteration as: 
\begin{equation}
    A_{k+1} = 2A_k - A_k A A_k, \quad \text{and} \quad A_0 = \alpha A. 
    \label{eq: newton}
\end{equation}
We now prove that $A_k$ finally converges to Moore-Penrose inverse of $A_{m\times m}$, if $\alpha$ is sufficiently small \citep{ben1966iterative}. 
\begin{proposition}
When $\alpha$ is sufficiently small, $A_{k+1}=2A_k-A_k A A_k$,  $A_k $ converges to $A^{\dagger}$.
\label{theo}
\end{proposition}
\major{The proposition is proved in Appendix \ref{suppsec:newton}.}
Though $\alpha =2/\|A\|_1^2$ which ensures good convergence behavior in Algorithm \ref{NR},
in practice, we find that using an alternative form gives more stable training and faster convergence.
Specifically, in $\|I-A\frac{2\beta^{n}}{\|A\|_1^2}\|_1\leq1$ where $\beta$ equals to $0.5$, we find the smallest $n_{i}$ that holds this inequality. Then, we initialize $\alpha$ as $\alpha=\frac{2\beta^{n_i}}{\|A\|_1^2}$.

The following proposition comes with the proof of Proposition \ref{theo}:
\begin{proposition}
\label{pro:AAkA-A}
$\|A A_k A - A\|$ and $\|A_k - A^{\dagger}\|$ decreases to $0$ monotonously, if $\alpha$ is sufficiently small.
\end{proposition}
\begin{proof}
Note that when we multiply $A$ on both sides of  \eqref{eq: iter}, the equation turns to be:
\begin{equation}
\begin{aligned}
        A-A A_{k+1}A &= A(A^{\dagger}-A_k)A(A^{\dagger}-A_k)A\\
        &=(AA^{\dagger}-AA_k)(A-A A_k A).
        \label{eq:norm}
\end{aligned}
\end{equation}
    Similarly norm both sides of  \eqref{eq:norm}, considering that $\|A A^\dagger - A A_{k}\|\rightarrow 0$ and $\|A A^\dagger - A A_{k}\|<1$ always holds, $\|A - A A_{k}A\|$ monotonically decreases to $0$. The inequality \eqref{converge} implies that $\|A_k - A^{\dagger}\|$ decreases to $0$ monotonously .
\end{proof}
Note although $\|A - A A_{k}A\|$ monotonically decreases to $0$, $\|A_k AA_k - A_{k}\|$ cannot be proved to be so yet.

\revision{
This ensures that our estimated inverse is sufficiently accurate for matrix decomposition, subject to that our SOFT attention is regularized.
In training stage, we find that bottleneck matrix $A$ is always non-singular in practice and its inverse $A^{-1}$ thus exists.
Therefore, the iteration can be avoided in back propagation because the differential of matrix inverse can be explicitly expressed as
\begin{equation}
\label{equ:inverse_back}
    \nabla_x \mathcal{L} = -Y^\top (\nabla_Y \mathcal{L}) Y^\top, 
\end{equation}
where $X\in \mathbb{R}^{m\times m}$ is an non-singular matrix and $Y$ is the inverse of $X$, \ie, $Y=X^{-1}$, $\nabla_Y \mathcal{L}$ is the gradient of loss $\mathcal{L}$ on $Y$ and $\nabla_X \mathcal{L}$ is the gradient of loss $\mathcal{L}$ on $X$.
This can accelerate the training,
as validated in Sec. \ref{setup}.
\major{The theoretical proof is shown in Appendix \ref{suppsec:newton}.}
}

\noindent {\bf Complexity: }
We summarize the complexity of SOFT in space and time.
For {\em time complexity},
it involves:
(1) Sampling: $\mathcal{O}(n d_e)$. 
(2) Calculating three decomposed matrices: $\mathcal{O}(nmd_e+mnd_e + m^2d_e) = \mathcal{O}(2mnd_e + m^2d_e)$;
(3) Moore-Penrose inverse: $\mathcal{O}(\mathcal{T}\times m^3) = \mathcal{O}(\mathcal{T}m^3)$, where $\mathcal{T}$ is the iteration steps. 
(4) All matrix multiplication: 
$\mathcal{O}(nm^2 + mnd_e + mnd_e) = \mathcal{O}(nm^2 +2mnd_e)$. The total time complexity is
$\mathcal{O}((d_e+4md_e+m^2)n+\mathcal{T}m^3 + d_em^2)$.
The {\em space complexity} is decided by four decomposed matrices with
$\mathcal{O}(n\times m) + \mathcal{O}(m\times m) + \mathcal{O}(m\times n) + \mathcal{O}(n\times d_e)=\mathcal{O}((2m+d_e)n+m^2)$.
As we keep $m$ ($m\ll n$) a fixed constant in our model,
both time and space complexity are $\mathcal{O}(n)$, 
making SOFT a linear self-attention.

\subsection{Attention normalization \label{sec:normalization}}

\rerevision{
Whilst the above SOFT formulation is competitive for image classification, transferring the pre-trained model to  downstream tasks with different input token sequence lengths
is limited.
We conduct analysis on the model's sensitivity against input perturbation.
As suggested in \citep{yoshida2017spectral},
all parts of a model should have small \textit{spectral norm} (\ie, matrix 2-norm or the largest singular value of a matrix).
Concretely, by the triangle inequality $\|XY\|_2 \leq \|X\|_2\|Y\|_2$
where $X$ and $Y$ are any real matrix, any part with large spectral norm could lead to significant error accumulation.
This also applies for self-attention matrix, 
since it is often symmetric and non-negative definite
and its spectral norm corresponds to the largest eigenvalue.
We provide more theoretical analysis below.
Note that both scale dot-product kernel (linear kernel) and Gaussian kernel are positive definite kernels, so they have non-negative eigenvalues.
}

\begin{proposition}
In scaled dot-product based softmax self-attention, assume $\lambda_1\geq \lambda_2\geq \cdots \lambda_n\geq 0$ are eigenvalues of self-attention matrix $S_{softmax}\in\mathbb{R}^{n\times n}$, then $\lambda_1 \leq 1$.
\label{theo:softmax_eigen}
\end{proposition}
\begin{proof}
We rewrite the softmax self-attention as
\begin{equation}
    S_{softmax} = D^{-1}A,
\end{equation}
where $A\in \mathbb{R}^{n\times n}$ is a real symmetric matrix, and $D=\text{diag}(A\mathbbm{1}_n)$. 
We consider the graph Laplacian $L$ of matrix $A$ defined by $L=D-A$, then the noramlized graph Laplacian can be expressed as
\begin{equation}
    L_{rw} = D^{-1}L = D^{-1}(D-A) = I-D^{-1}A,
\end{equation}
where, $I$ is an identity matrix.
According to \citep{von2007tutorial}, $L_{rw}$ is a real semi-definite matrix, so all the eigenvalue of $L_{rw}$ is non-negative.
Therefore, eigenvalues of $I-D^{-1}A$ are non-negative, which leads to the fact that eigenvalues $\lambda_1,\cdots, \lambda_n$ of $D^{-1}A$ is less or equal to 1.
\end{proof}
\extension{
This proposition means that softmax normalization is useful in restricting the range of self-attention matrix's eigenvalues to $[0, 1]$ and eventually the effect of error accumulation.
This is an important role \textit{softmax} operation plays in improving generalizability with standard self-attention. 
However, this is not the case for our softmax-free self-attention as formulated above.
}
\begin{proposition}
In Gaussian kernel-based self-attention, if $\lambda_1\geq \lambda_2\geq \cdots \lambda_n\geq 0$ are eigenvalues of self-attention matrix $S_{gaussian}\in\mathbb{R}^{n\times n}$, then $\lambda_1 \leq n$.
\label{theo:gaussian_eigen}
\end{proposition}
\begin{proof}
\begin{equation*}
    \sum_{i=1}^n\lambda_i = \text{Tr}(A)=n, 
\end{equation*}
since the diagonal elements of Gaussian kernel-based self-attention is always 1.
Therefore, with the fact that all the eigenvalues are positive, we have $\lambda_1\leq n$.
\end{proof}

\extension{
A larger upper bound of eigenvalue with our proposed self-attention could thus lead to less generalizability,
due to a trend of higher error accumulation.
Specifically, for a softmax-free self-attention matrix, $\hat{S} = P^\top A^\dag P$, we have }
\begin{equation}
     \|\hat{S}\|_2 = \|P^\top A^\dag P\|_2 \leq \|P\|_2^2 \|A^\dag\|_2.
\end{equation}
\extension{
For the bottleneck matrix of softmax-free attention $A\in\mathbb{R}^{m\times m}$, we assume it is $k$-connected ($k<<m$), \ie, there are $k$ fields disconnected with each other.
This is because a field represents a semantic part of an image and the number is often small.
}
\begin{proposition}
Assume the bottleneck matrix of softmax-free attention , $A\in\mathbb{R}^{m\times m}$ is $k$-connected. If $\lambda_1\geq \lambda_2\geq \cdots \lambda_m \geq 0$ are eigenvalues of $A^\dag$, then $\lambda_1 = \mathcal{O}(m^2)$ and $\|A^\dag\|_2 = \mathcal{O}(m^2)$.
\label{theo:inverse_eigen}
\end{proposition}
\major{The proposition is proved in Appendix \ref{suppsec:attn_norm}.}
\extension{
This proposition indicates that 
$\|\hat{S}\|_2$ is proportional quadratically with 
the length $m$ of bottleneck token sequence.
It implies that the above SOFT formula would be limited to the applications with short bottleneck token sequences (\eg, image classification).
}

\begin{proposition}
For the bottleneck matrix of SOFT self-attention $A\in\mathbb{R}^{m\times m}$, we have
\begin{equation}
    \|D^{-1/2} A^{\dagger} D^{-1/2}\|_2 = \mathcal{O}(m), 
\end{equation}
where $D = \text{diag}(A\mathbbm{1}_m)$ and $\mathbbm{1}_m$ is an all one $m$-D vector.
\label{theo:norm_eigen}
\end{proposition}
\extension{
\begin{proof} 
\begin{equation*}
    \|D^{-1/2}A^\dag D^{-1/2}\|_2 \leq \|D^{-1/2}\|_2^2 \|A^\dag\|_2 = \|D^{-1}\|_2\|A^\dag\|_2, 
\end{equation*}
also, $D$ is a diagonal matrix, 
\begin{equation*}
    \|D^{-1}\|_2\|A^\dag\|_2 = \|D^{-1}A^\dag\|_2\ = \|A_n^\dag\|^2 = \mathcal{O}(m)
\end{equation*}
\end{proof}
This proposition suggests that symmetric normalization can lower the largest eigenvalue of $A^\dag$, reducing the spectral norm of $\hat{S}$.}

\noindent{\bf Attention normalization: }
\extension{
In light of the above theorem, we further introduce the normalization of SOFT as: 
} 
\begin{equation} \label{eq:reg_norm_attn}
    \hat{S} = \text{exp}\left(Q\ominus \widetilde{K}\right)D^{-\frac{1}{2}} \left(\text{exp}\left(\widetilde{Q}\ominus \widetilde{K}\right)\right)^{\dagger}D^{-\frac{1}{2}}
    \text{exp}\left(\widetilde{Q}\ominus K\right),
\end{equation}
\extension{
where $D=\text{diag}\left(\text{exp}\left(\widetilde{Q}\ominus \widetilde{K}\right) \mathbbm{1}_m\right)$.
This yields a tiny increase of $\mathcal{O}(m^2)$ in the computational complexity, which can be further reduced to $\mathcal{O}(\log{m})$ on GPU by parallel reduction algorithm \citep{cheng2014professional}.}
\rerevision{This discovery unleashes new possibilities for SOFT and paves the way for its expanded usage, particularly in applications that demand dense information reasoning (such as object detection and semantic segmentation). 
This results in our SOFT++ formulation.
}

\subsection{Instantiations}

\label{instantiations}
Figure \ref{fig:Structure} shows how our proposed {\em softmax-free self-attention} block ({\bf SOFT block}) can be implemented in a neural network.
We replace the self-attention block with our SOFT block in the traditional Transformer, that is, we stack a SOFT block with a feed forward residual block \citep{dosovitskiy2020image} to form a {\em softmax-free Transformer} layer ({\bf SOFT layer}).

Focusing on the general image recognition tasks,
we integrate our SOFT layer into the recent pyramidal Transformer architecture \citep{wang2021pyramid} to form our final model {\bf SOFT}.
Further, several improvements are introduced in patch embedding (\ie, tokenization).
Specifically, unlike \citep{wang2021pyramid} that uses a combination of non-overlapping convolution and layer normalization \citep{ba2016layer}, we adopt a stack of overlapping convolutions, batch normalization \citep{ioffe2015batch} and ReLU non-linearity.
Concretely, the $\tt STEM$ is implemented by 3 units of $\tt 3x3\; Conv{\to}BN{\to}ReLU$, with the stride of 2, 1, 2 respectively. 
Then, one such unit is applied to each of three following down-sampling operations with stride of 2 in the multi-stage architecture.

The architecture hyper-parameters of SOFT are:
    $d$: the input channel dimension of SOFT layer.
    $d_e$: the embedding dimension of tokens in SOFT block. In practice, we set $d_e=d$. 
    $h$: the head number of SOFT block.
    $d_h$: the channel dimension of each head and $d_h=d_e/h$.
    $n$: the input token sequence length of a SOFT block.
    $m$: the bottleneck token sequence length of SOFT block.
    $sp$: the sampling ratio of token sequence length sampling, which is the ratio between input token sequence length and the bottleneck token sequence length.
    $e$: the expansion ratio of the 2-layer feed forward block.
In SOFT, for all the stages
we set $d_h=32$, $e=4$
and $m=49$,
$sp$ varies in each stage according to the input token sequence length. 
Table \ref{table:arch-spec} details the family of our SOFT configurations with varying capacities (depth and width). 

%% file: figure/structure.tex
\begin{figure*}[t]\centering
\includegraphics[width=1.0\linewidth]{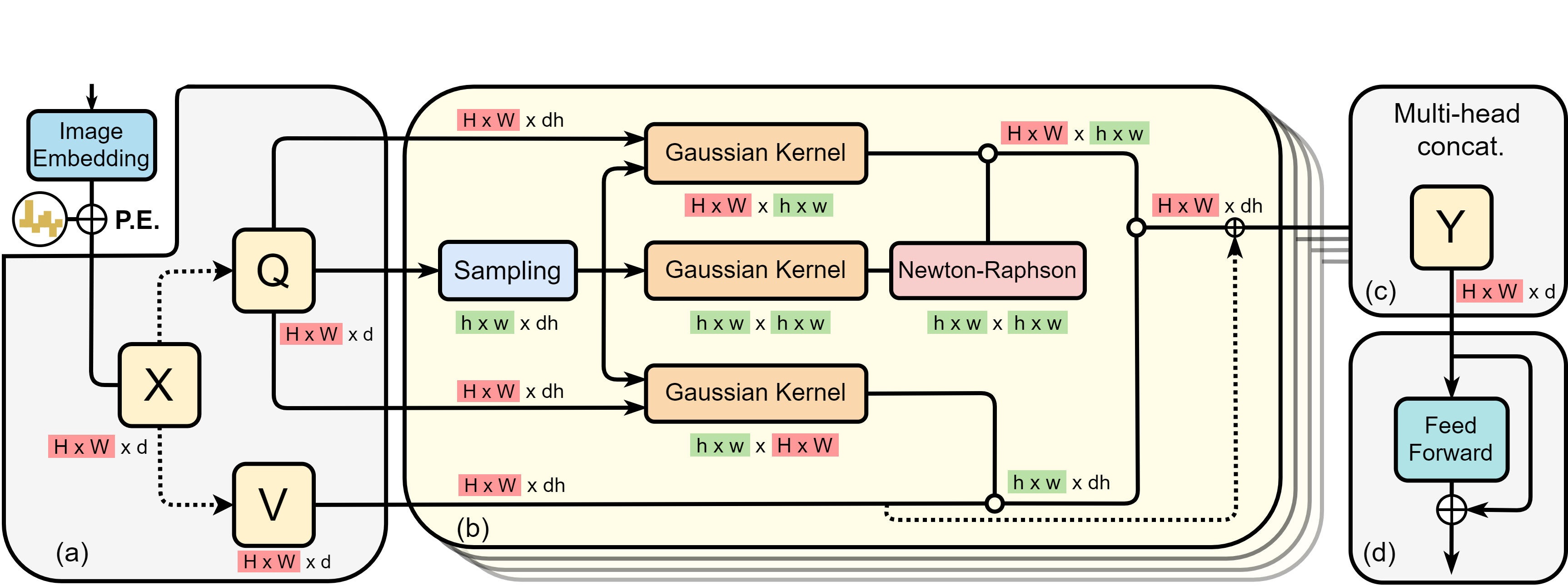}
\caption{Schematic illustration of the proposed {\em softmax-free self-attention} (SOFT) method. 
\texttt{P.E.}: Position embedding.
Dash lines: linear projection. 
\texttt{dh}: the hidden dim of each attention head. 
$\circ$ denotes the matrix dot product.}
\label{fig:Structure}
\end{figure*}

%% file: algorithm/soft.tex
\begin{algorithm}
\SetAlgoLined
\KwIn{$Q\in \mathbb{R}^{n\times d_e}$, sampling function $f_s$}
\KwSty{Sampling}  $\ \widetilde{Q} \leftarrow f_s(Q)$ \;
$A\leftarrow\text{exp}(\widetilde{Q}\ominus \widetilde{Q})$, $P\leftarrow\text{exp}(\widetilde{Q}\ominus Q)$\;
$\hat{S}\leftarrow P^\top \text{NR}(A) P$\;
\KwOut{$\hat{S}$}
 \caption{SOFT: Softmax-free attention}
 \label{soft}
\end{algorithm}

%% file: algorithm/newton_raphson.tex
\begin{algorithm}
\SetAlgoLined
\KwIn{$A\in \mathbb{R}^{m\times m}$, and $\mathcal{T}\in \mathbb{Z}^+$}{$\alpha= 2/\|A\|_1^2$.Initialize $A_0\leftarrow \alpha A$}\;
\For{$k$ from $1$ to $\mathcal{T}$ }{
$A_{k}\leftarrow 2A_{k-1}-A_{k-1}AA_{k-1}$
}
\KwOut{$A_\mathcal{T}$}
 \caption{NR: Newton-Raphson iteration}
 \label{NR}
\end{algorithm}

%% file: sections/4_experiments.tex
\input{table/architecture}

\input{table/linearization}

\input{table/imagenet}
\input{figure/fps_vs_top1}
\input{table/retina_od}
\input{table/maskrcnn_inst}

\section{Experiments}
\label{experiment}
\subsection{Image classification}
\label{setup}

\noindent{\bf Dataset:}
We evaluate the proposed SOFT and \rerevision{SOFT++} on the ILSVRC-2012 ImageNet-1K dataset \citep{deng2009imagenet} with 1.28M training images and 50K validation images from 1,000 classes.
Following the common practice, we train a model on the training set and evaluate on the validation set.

\noindent{\bf Metrics:}
For model performance, the top-1 accuracy on a single crop is reported.
To assess the cost-effectiveness, we also report the model size
and floating point operations (\ie, FLOPs).

\noindent{\bf Implementation details:}
We use the code base \citep{rw2019timm} with the default setting to train and test all the models.
Specifically, we use weight decay of 0.05 and 10 epochs of linear warm-up.
We conduct 300 epochs training with an $\tt AdamW$ optimizer 
and decreasing learning rate with the cosine annealing schedule.
During training, random flipping, mixup \citep{zhang2017mixup} and cutmix \citep{yun2019cutmix} are adopted for data augmentation. 
Label smoothing \citep{szegedy2016rethinking} is used for loss calculation. 
All our variants are trained with a batch size of 1024 on 32G NVIDIA V100 GPUs.
We also implement our method using the Mindspore \citep{mindspore}.

\noindent{\bf Comparison with existing linear Transformers: \label{sec:competitiors}}We compare our method with three existing linear Transformer models:
Linformer \citep{wang2020linformer}, 
Performer \citep{choromanski2020rethinking},
Nystr{\"o}mformer \citep{xiong2021nystr} in terms of model complexity and accuracy.

Two experimental settings are adopted. Under the first setting,  for all methods we use the same $\tt Tiny$ (Table \ref{table:arch-spec}) architecture
for a fair comparison. 
That is, we replace the core self-attention block in SOFT\rerevision{/SOFT++} with each baseline's own attention block with the rest of the architecture unchanged.
Note that the {\em spatial reduction} module of \citep{wang2021pyramid}
is a special case of Linformer \citep{wang2020linformer}.
We set the reduction ratio to be identical to ours.
With the same uniform sampling idea,
we replace the 1D window averaging of Nystr{\"o}mformer \citep{xiong2021nystr} (for NLP tasks) with 2D average pooling (for images). 
The downsampling ratio remains identical to ours.
It is also worth mentioning that there is no official code released for Reformer \citep{kitaev2020reformer} and the local Sensitive Hash (LSH) module has strict requirements on the length of input tokens.  We thus do not include this method in our comparison.

From Table \ref{tab:linearizationmethods} we can make the following observations:
(i) Linear Transformer methods substantially reduce the memory and FLOPs while maintain similar parameter size comparing to the Transformer on the $\tt Tiny$ architecture;
(ii) Our approach SOFT\rerevision{/SOFT++} achieve the best classification accuracy among all the linearization methods.
(iii) Our inference speed is on-par with other compared linear Transformers and our training speed is slightly slower than Nystromformer and both are slower than Performer and Linformer.
Note that the slow training speed of our model is mostly due to the Newton-Raphson iteration which can only be applied sequentially for ensuring the accuracy of Moore-Penrose inverse.
In summary, due to the on-par inference speed we consider the training cost increase is a price worth paying for our superior accuracy.

Under the second setting,  we focus on  the memory efficiency of SOFT against the baselines.
Here we follow the ViT \citep{dosovitskiy2020image} network structure, stacking 12 attention layers with hidden dimension $d=384$, heads $h=12$, bottleneck token sequence length $m=49$.
Different attention blocks from the three
linearized Transformer variants, Linformer \citep{wang2020linformer}, Performer \citep{choromanski2020rethinking},
and
Nystr{\"o}mformer \citep{xiong2021nystr} are studied.
For each Transformer variant, we adjust its token sequence length $n$ in a linear increment. 
Specifically, we use a token sequence length of $784\times p$ where $p=1,2,3,4,5,6,7,8$ and set batch size 1 to verify whether the memory consumption increases “quadratically” or “linearly”.
Figure \ref{fig:formercomparison} shows all compared transformer variants including our SOFT indeed have a linear memory usage complexity. 
This is in contrast with the standard Transformer which cannot cope with long token sequences with a quadratic complexity.

\noindent{\bf Comparison with state-of-the-art CNNs and ViTs: }
We compare with state-of-the-art alternatives and report the top-1 accuracy on the ImageNet-1K validation set.
FLOPs are calculated at batch size 1. 
From Figure \ref{fig:paramettop1} and Table \ref{tab:classification}, the following observations are made:
(i) Overall, ViT and its variants yield better classification accuracy over CNNs.
(ii) We achieve the best performance among the recent pure vision Transformer based methods including  ViT \citep{dosovitskiy2020image} and DeiT \citep{touvron2021training}, as well as the state-of-the-art CNN RegNet \citep{radosavovic2020designing}.
(iii) Our SOFT and \rerevision{SOFT++} outperform the most similar (in architecture configuration) Transformer counterparts PVT \citep{wang2021pyramid} at all variants. Since the attention module is the main difference, this validates directly the effectiveness of our model.
(iv) We can also beat the latest ViT variants Twins \citep{chu2021twins} which is designed to address the efficiency limitation of ViT. We have done so with less parameters and fewer float point computation.

\input{figure/attention_heatmap}

To gain insights into how attention is learned using our SOFT and the alternatives, Figure \ref{fig:attention_heat} shows the attention masks. 
For each model, we visualize the first two attention heads.
It is evident that 
SOFT exhibits robustness and versatility in capturing local and long distance relations among pixels. 
We note that, although SOFT is trained for object categorization (ImageNet \citep{deng2009imagenet}), it seems to be able to learn both semantic concepts shared across instances in the same category and instance specific features. For instance, in the bottom-right example of a bird class, one attention head focuses on the  black bird only, while the other attends to both birds in the image. 
More examples are given in 
Appendix \ref{sec:attn_vis}.

\noindent{\bf Throughputs comparison: }
\major{
Figure \ref{fig:fps_top1} illustrates that our model achieves a superior balance between speed and performance compared to the Swin Transformer \citep{liu2021swin}, and maintains a comparable balance with the current CNN state-of-the-art, ConvNext \citep{liu2022convnet}.
As detailed in Table \ref{tab:linearizationmethods}, we employ a series-computed Newton iteration method to ensure numerical accuracy. While this approach slightly reduces speed, it preserves the accuracy of our model.
}

\noindent{\bf Comparison between SOFT and SOFT++}
\rerevision{As shown in Table \ref{tab:classification}, SOFT++ performs equally well as SOFT for small models, and outperforms the preliminary version for larger models. 
Further, Table \ref{tab:seqlen_cls} 
validates the superiority of SOFT++ over SOFT in processing
long token sequences in model deployment.
}
\input{table/seqlength_classification}

\subsection{Object detection on COCO}
\noindent{\bf Dataset: }
\extension{
We evaluate the object detection performance of our SOFT++ on the COCO benchmark \citep{lin2014microsoft} including the \texttt{train2017} (118k images) and \texttt{val2017} (5k images) sets.}

\noindent{\bf Implementation details: }
\extension{
We consider two representative detectors: RetinaNet \citep{lin2017focal} and 
Mask R-CNN \citep{he2017mask}.
We use \texttt{AdamW} optimizer with base learning rate of $1\times 10^{-4}$ and weight decay of $0.01$. 
We train all the model with batch size 16 on 8 V100 GPUs. 
Following the practices of MMDetection \citep{chen2019mmdetection}, we adopt the 1$\times$ and 3$\times$ training schedules. 
In the training stage without multi-scale, images are resized to make the shorter sides at 800 pixels and the longer sides no exceeding 1333 pixels.
In the training stage with multi-scale, the shorter sides of images are randomly resized to between 480 to 800.}

\noindent{\bf Comparison with state-of-the-art CNNs and ViTs: } \extension{
We compare with the state-of-the-art alternatives on RetinaNet \citep{lin2017focal} in Table \ref{tab:retina} and Mask R-CNN \citep{he2017mask} in Table \ref{tab:maskrcnn}.
Our SOFT++ outperforms the CNN ResNet \citep{he2016deep} and Transformer counterparts PVT \citep{wang2021pyramid} across all complexity groups in both object detection and instance segmentation, validating the superior trade-off between model complexity and performance by our design.}

\subsection{Semantic segmentation}
\noindent{\bf Dataset: }
\rerevision{We evaluate the semantic segmentation performance of our SOFT++ on ADE20K \citep{zhou2019semantic} and Cityscapes \citep{cordts2016cityscapes}.
}

\noindent{\bf Implementation details: }
\rerevision{
UperNet \citep{xiao2018unified} is used as the framework. 
\texttt{AdamW} optimizer with $6\times 10^{-4}$ learning rate is applied to train ADE20K for 160k iterations and Cityscapes for 40k iterations.
Images are cropped as $512\times 512$ for ADE20K and $ 768 \times 768 $ for Cityscapes during training.
Multi-scale training and test time augmentation are not used.
}

\noindent{\bf  Comparison with state-of-the-art CNNs and ViTs: }
\rerevision{}{
Table \ref{tab:soft_segmentation} shows that SOFT++ surpasses ResNet clearly, and the newest designed vision transformer Swin \citep{liu2021swin} and CovNext \citep{liu2022convnet} slightly under the similar parameters and float point computation cost.}
\input{table/soft_segmentation}

\subsection{Ablation studies}
\label{sec: abl}
{\bf Pyramidal architecture: } 
Unlike the earlier non-pyramidal vision Transformers (\eg, ViT \citep{dosovitskiy2020image}), most recent pyramidal (multi-scale) Transformers (\eg, PVT \citep{wang2021pyramid}) use convolution layers to reduce the spatial resolution (\ie, token sequence length) between stages. 
In this study, we ablate SOFT++ with a pyramidal architecture (our default SOFT++-$\tt Small$), SOFT++ w/o a pyramidal architecture and DeiT-S \citep{touvron2021training} (no pyramidal architecture either). 
We replace the Transformer layer with a SOFT++ layer to get SOFT++ w/o a pyramidal architecture. 
Note all three variants have similar parameters and FLOPs.
Table \ref{tab:pyramidal-overlapping}a shows that the conv-based pyramidal architecture is clearly superior to a non-pyramidal design, and our non-pyramidal counterpart is even slightly better than DeiT-S \citep{touvron2021training} whilst enjoying linear complexity.
\input{table/symmetric}
\input{table/bottleneck}
\input{table/sampling}

\noindent\major{{\bf Symmetric kernels: }
Tsai et al. \citep{tsai2019transformer} conducted a comprehensive analysis of Transformers, revealing minimal performance disparities between asymmetric and symmetric kernels in NLP tasks, such as neural machine translation (NMT) and sequence prediction (SP). 
In contrast, our research, detailed in Table \ref{tab:abl_sym}, extends this investigation to the domain of computer vision, specifically focusing on image classification tasks. 
Our findings indicate that while symmetric kernels do have a marginal detrimental effect in computer vision tasks, this impact is relatively limited. 
Therefore, maintaining a symmetric kernel represents a balanced approach for enabling effective decomposition in our model.}

\noindent{\bf Bottleneck token sequence length: }
\input{figure/newtoniter_rank2}
In this study, we examine how the bottleneck token sequence length $m$, sampled from $n$ tokens, influences the model's performance. 
We change the bottleneck token sequence length in all stages to $36,49,64,81$.
Table \ref{tab:bottleneck token} shows that longer bottleneck token would increase the memory cost and the computational overhead.
$m=49$ seems to give the best trade-off between the performance and computational overhead. The memory usage is measured with the batch size of 128.

\input{table/pyramidal-overlapping}
\noindent{\bf Token sampling: }
The sampling function in SOFT++ can assume different forms.
\textbf{\em Convolution:} The sequence $Q\in \mathbb{R}^{n\times d_e}$ is first reshaped to a feature map $\mathbb{R}^{H\times W\times d_e}$. 
$r\times r$ convolution kernel with stride of $r$ is applied for downsampling, where $r=\sqrt{sp}$. 
The output channel size is also kept and no bias is used.
At last, the feature map is reshaped back to the sequence.
\textbf{\em Average pooling:} 
using a $r\times r$ kernel and $r$ stride, where $r=\sqrt{sp}$.
\textbf{\em Random sampling:} $m$ tokens are randomly picked from $n$ tokens. 
\textbf{\em Biased sampling: }
We pick $m$ tokens
with a biased policy. 
Here, the first $m$ tokens are picked. 
Table \ref{tab:samplingmethods} shows that both convolution yields the best performance.
Biased sampling can miss the most salient samples, and there is no guarantee that random sampling can keep the uniformity of the chosen samples. 
This result thus justifies the choice of using convolution in SOFT++.

\noindent{\bf Overlapped convolution: } We ablate SOFT++ with overlapped convolution (our default choice, same as many recent works) and with non-overlapped convolution in our $\tt Tiny$ configuration.
Table \ref{tab:pyramidal-overlapping}b shows that 
overlapped convolution is a better choice.
Our non-overlapped convolution variant still outperforms  PVT \citep{wang2021pyramid} with the same non-overlapped convolution by a clear margin.

\noindent\textbf{Newton-Raphson's convergence:}
We study how many iterations the Newton-Raphson method needs to converge when computing the Moore-Penrose inverse $A^{\dagger}$.
We use $\|AA_kA-A\|_p/\|A\|_p$ with $p=2$ (see Proposition \ref{pro:AAkA-A}) as the convergence metric to quantify the difference between $A_k$ and $A^\dagger$.
Figure \ref{fig:newtoniterationrelerr} shows that
our approximation converges within
20 iterations across all stages.

\input{table/lra}
\noindent {\bf Effect of attention normalization:}
\rerevision{
Our attention normalization enables the model to perform 
more challenging vision tasks such as object detection and segmentation.
This is because 
the leading eigenvalue of Gaussian kernel self-attention is conditioned on the token sequence length (Proposition \ref{theo:gaussian_eigen}), as shown in Figure \ref{fig:eigen}.
And this effect can be clearly mitigated by our normalization scheme, as indicated in
Figure \ref{fig:eigen_vs_m}.
}
\input{figure/symmetric_norm_eigen}

\subsection{Additional experiments on NLP tasks}
We compare our method with previous linear counterparts on four tasks of the Long Range Arena (LRA) \citep{tay2020long} benchmark, including Listops \citep{nangia2018listops}, byte-level IMDb reviews text classification \citep{maas2011learning}, byte-level document retrieval \citep{radev2013acl}, and image classification on sequences of pixels \citep{krizhevsky2009learning}. 

\noindent\textbf{Implementations.}
We use the Pytorch version of LRA \citep{tay2020long} benchmark, provided by \citep{xiong2021nystr}. 
For the evaluation protocol, we strictly follow \citep{tay2020long,xiong2021nystr}.
We omit the Pathfinder(1K) task as we cannot replicate the result of Nystr{\"o}mformer \citep{xiong2021nystr}.
For our SOFT, we simply use the average pooling with window size 4, stride 4 to sample the bottlenecks.
We follow the configurations of \citep{xiong2021nystr} with 2 layers, 64 and 128 hidden dimension respectively, and 2 attention heads. 
Table \ref{tab:lra} shows that our SOFT outperforms both the standard and alternative efficient Transformers on three out of four tasks, as well as the average performance. 

%% file: table/architecture.tex
\begin{table*}[t]

\centering
\caption{Architecture specifications of SOFT variants.
\textit{sp.}: sampling ratio. 
\textit{-d}: the hidden dimension. 
\textit{-h}: the number of heads in the self-attention block.
\textit{C33-BN-ReLU}: three 3x3 Conv-BN-ReLU, with the stride of 2, 1, 2 respectively.
\textit{C31-BN-ReLU}: one 3x3 Conv-BN-ReLU, with a stride of 2.
}
\renewcommand{\arraystretch}{1.5} 
\begin{tabular*}{\textwidth}{@{\extracolsep\fill}ccccc}
\hline
 & Tiny  & Small & Medium & Large \\
\hline

\hline
\multirow{3}{*}{Stage 1} & \multicolumn{4}{c}{C33-BN-ReLU, 64-d}  \\

& $\begin{bmatrix}\text{sp. 8x8,}\\\text{64-d, 2-h}\end{bmatrix}$ x 2   & $\begin{bmatrix}\text{sp. 8x8,}\\\text{96-d, 3-h}\end{bmatrix}$ x 2    & $\begin{bmatrix}\text{sp. 8x8,}\\\text{96-d, 3-h}\end{bmatrix}$ x 2   & $\begin{bmatrix}\text{sp. 8x8,}\\\text{128-d, 4-h}\end{bmatrix}$ x 2 \\

\multirow{3}{*}{Stage 2} & \multicolumn{4}{c}{C31-BN-ReLU, 128-d}  \\

& $\begin{bmatrix}\text{sp. 4x4,}\\\text{128-d, 4-h}\end{bmatrix}$ x 2  & $\begin{bmatrix}\text{sp. 4x4,}\\\text{192-d, 6-h}\end{bmatrix}$ x 2 & $\begin{bmatrix}\text{sp. 4x4,}\\\text{192-d, 6-h}\end{bmatrix}$ x 2 & $\begin{bmatrix}\text{sp. 4x4,}\\\text{256-d, 8-h}\end{bmatrix}$ x 2 \\

\multirow{3}{*}{Stage 3} & \multicolumn{4}{c}{C31-BN-ReLU, 256-d}  \\

& $\begin{bmatrix}\text{sp. 2x2,}\\\text{320-d, 10-h}\end{bmatrix}$ x 5 & $\begin{bmatrix}\text{sp. 2x2,}\\\text{384-d, 12-h}\end{bmatrix}$ x 5 & $\begin{bmatrix}\text{sp. 2x2,}\\\text{384-d, 12-h}\end{bmatrix}$ x 18  & $\begin{bmatrix}\text{sp. 2x2,}\\\text{512-d, 16-h}\end{bmatrix}$ x 18\\

\multirow{3}{*}{\begin{tabular}[c]{@{}c@{}}Stage 4\\ w. cls token\end{tabular}} & \multicolumn{4}{c}{C31-BN-ReLU, 512-d}  \\

& $\begin{bmatrix}\text{sp. 1x1,}\\\text{512-d, 16-h}\end{bmatrix}$ x 2 & $\begin{bmatrix}\text{sp. 1x1,}\\\text{768-d, 24-h}\end{bmatrix}$ x 2  & $\begin{bmatrix}\text{sp. 1x1,}\\\text{768-d, 24-h}\end{bmatrix}$ x 2 & $\begin{bmatrix}\text{sp. 1x1,}\\\text{1024-d, 32-h}\end{bmatrix}$ x 2 \\
\hline
\end{tabular*}
\normalsize

\label{table:arch-spec}
\end{table*}

%% file: table/linearization.tex
\begin{table*}[htb]
\label{tab:linearizationmethods}

\renewcommand{\arraystretch}{1.5} 
\caption{Comparison of different linear/efficient transformer variants on ImageNet~\citep{deng2009imagenet}, based on our multi-stage Tiny configuration (see Table \ref{table:arch-spec}). 
The memory usage is measured with the batch size of 1024 which is our standard training setting.
Transformer is tested at a batch size of 256, which is the maximal number possible with the GPU resource at our disposal. 
Throughput is in format as $\text{Train throughput}\ /\ \text{inference throughput}$.
}
\begin{tabular*}{\textwidth}{@{\extracolsep\fill}lccccc}
\hline

Methods & Memory & Params & FLOPs & Throughput (img/s) & Top-1 \%\\

\hline

\hline

Transformer~\citep{vaswani2017attention} & 19.0GB$\dagger$ & 13M & 3.9G & 1073 / 3240 & 80.0\\
Linformer~\citep{wang2020linformer} & 11.7GB & 13M & 1.9G & 2767 / 3779 & 78.2 \\
Performer~\citep{choromanski2020rethinking} & 15.0GB & 13M & 2.2G & 2037 / 3657 & 76.1\\
Nystr{\"o}mformer~\citep{xiong2021nystr} & 17.2GB & 13M & 2.0G & 1891 / 3518 & 80.1\\
\textbf{SOFT} & 15.8GB & 13M & 1.9G & 1730 / 3436& \textbf{80.9}\\
\textbf{SOFT++} & 15.8GB & 13M & 1.9G & 1730 / 3436& \textbf{80.9}\\

\hline

\end{tabular*}
\end{table*}

%% file: table/imagenet.tex
\begin{table*}[!ht]
\caption{Evaluation results on ILSVRC-2012 ImageNet-1K ~\citep{deng2009imagenet} $\tt validation$ set. 
We report the results
using the input size of 224x224 pixels center cropped from resized images with 256x256 pixels. 
$\tt M.S. Out.$ stands for whether the model is designed for multi-scale output.
$\dagger$: Corrected FLOPs by
taking into account the cost of attention matrix multiplication
overlooked in the origin paper.
}
\renewcommand{\arraystretch}{1.5} 
\begin{tabular*}{\textwidth}{@{\extracolsep\fill}lcccccc}

\hline
Model  & Style & Resolution  &M.S. Out.? & Params & FLOPs & Top-1 \%.\\
\hline

\hline
ResNet-18~\citep{he2016deep} & Convolution & $224^{2}$  & \checkmark & 11M & 1.9G & 69.8 \\
PVT-Tiny~\citep{wang2021pyramid} & Transformers  & $224^{2}$ & \checkmark   & 13M & 1.9G\dag & 75.1\\
ConViT-Ti+ ~\citep{d2021convit} & Transformers& $224^2$ & \XSolidBrush & 10M & 2.0G & 76.7  \\
Coat-Lite Mini~\citep{xu2021co} & Transformers &  $224^{2}$ & \checkmark  & 11M &2.0G &78.9\\
LambdaNets-50~\citep{bello2021lambdanetworks} & Transformers &  $224^{2}$ & \checkmark &  16M &- &78.9\\
 ViP-Ti ~\citep{sun2021visual} & Transformers& $224^2$  & \checkmark& 13M & 1.7G & 79.0  \\
\rerevision{\textbf{SOFT-Tiny}} & \rerevision{SOFT} & \rerevision{$224^{2}$}  & \rerevision{\checkmark}  & \rerevision{13M} & \rerevision{1.9G} & \rerevision{80.9}\\
\rerevision{\textbf{SOFT++-Tiny}} & \rerevision{SOFT++} & \rerevision{$224^{2}$}  & \rerevision{\checkmark}  & \rerevision{13M} & \rerevision{1.9G} & \rerevision{\textbf{80.9}}\\

ResNet-50~\citep{he2016deep} & Convolution & $224^{2}$  & \checkmark & 25M & 4.1G & 78.5 \\
PVT-Small~\citep{wang2021pyramid}  & Transformer &  $224^{2}$ & \checkmark & 24M & 4.0G\dag & 79.8 \\
Swin-T~\citep{liu2021swin} & Transformer &  $224^{2}$ &  \checkmark & 29M & 4.5G & 81.3 \\
Twins-SVT-S~\citep{chu2021twins} & Hybrid &  $224^{2}$ &  \checkmark & 24M & 3.7G & 81.7 \\
CoAtNet-0~\citep{dai2021coatnet}& Hybrid& $224^2$ & \checkmark & 25M & 4.2G & 81.6\\
 DeiT III-S~\citep{touvron2022deit}& Transformer& $224^2$ & \XSolidBrush & 22M & 4.6G & 81.4\\
SwinV2-T~\citep{liu2022swin}& Transformer & $256^2$ & \checkmark & 29M & 4.5G & 81.7  \\
ConvNext-T~\citep{liu2022convnet}& Hybrid & $224^2$ & \checkmark & 29M & 4.5G & 82.1  \\

\rerevision{\textbf{SOFT-Small}} & \rerevision{SOFT} & \rerevision{$224^{2}$}  & \rerevision{\checkmark}  & \rerevision{27M} & \rerevision{4.5G} & \rerevision{82.5}\\
\rerevision{\textbf{SOFT++-Small}} & \rerevision{SOFT++} & \rerevision{$224^{2}$}  & \rerevision{\checkmark}  & \rerevision{27M} & \rerevision{4.5G} & \rerevision{\textbf{82.6}}\\

ResNet-101~\citep{he2016deep} & Convolution & $224^{2}$ & \checkmark & 44M & 7.9G & 79.8\\
PVT-Medium~\citep{wang2021pyramid} & Transformer & $224^{2}$  & \checkmark & 44M & 7.0G\dag & 81.2 \\
ViT-Small/16~\citep{dosovitskiy2020image} & Transformer & $224^{2}$ & \XSolidBrush & 48M & 9.9G & 80.8\\
Swin-S~\citep{liu2021swin} & Transformer &  $224^{2}$ &  \checkmark & 50M & 8.7G & 83.0 \\
ConvNext-S ~\citep{liu2022convnet}& Hybrid & $224^2$  &  \checkmark& 50M & 8.7G & 83.1  \\
 CoAtNet-1~\citep{dai2021coatnet}& Transformer& $224^2$ &  \checkmark & 42M & 8.4G & 83.3\\
SwinV2-S ~\citep{liu2022swin}& Transformer & $256^2$ &  \checkmark & 50M & 8.7G & 83.6  \\
\rerevision{\textbf{SOFT-Medium}} & \rerevision{SOFT} & \rerevision{$224^{2}$}  & \rerevision{\checkmark}  & \rerevision{48M} & \rerevision{8.7G} & \rerevision{83.2}\\
\rerevision{\textbf{SOFT++-Medium}} & \rerevision{SOFT++} & \rerevision{$224^{2}$}  & \rerevision{\checkmark}  & \rerevision{48M} & \rerevision{8.7G} & \rerevision{\textbf{83.7}}\\

CaiT-S36\citep{touvron2021going} & Transformer & $224^{2}$ & \checkmark & 88M & 13.9G & 83.3\\
Swin-B\citep{liu2021swin} & Transformer & $224^{2}$ & \checkmark & 88M & 15.4G & 83.3\\
Twins-SVT-L~\citep{chu2021twins} & Hybrid &  $224^{2}$ &  \checkmark & 99M & 14.8G & 83.3 \\
DeiT III-B~\citep{touvron2022deit} & Transformer & $224^2$ &  \XSolidBrush & 87M & 15.5G & 83.8\\
 ConvNext-S ~\citep{liu2022convnet} & Hybrid & $224^2$ &  \checkmark & 89M & 15.4G & 83.8  \\
 CoAtNet-2~\citep{dai2021coatnet} & Hybrid & $224^2$ &  \checkmark & 75M & 15.7G & 84.1\\
SwinV2-B ~\citep{liu2022swin} & Transformer & $256^2$ &  \checkmark & 88M & 15.4G & 84.1 \\
\rerevision{\textbf{SOFT-Large}} & \rerevision{SOFT} & \rerevision{$224^{2}$}  & \rerevision{\checkmark}  & \rerevision{85M} & \rerevision{15.4G} & \rerevision{83.6}\\
\rerevision{\textbf{SOFT++-Large}} & \rerevision{SOFT++} & \rerevision{$224^{2}$}  & \rerevision{\checkmark}  & \rerevision{85M} & \rerevision{15.4G} & \rerevision{\textbf{84.1}}\\

\hline
\end{tabular*}

\label{tab:classification}
\end{table*}

%% file: figure/fps_vs_top1.tex
\begin{figure}[!htb]\centering
\includegraphics[width=\linewidth]{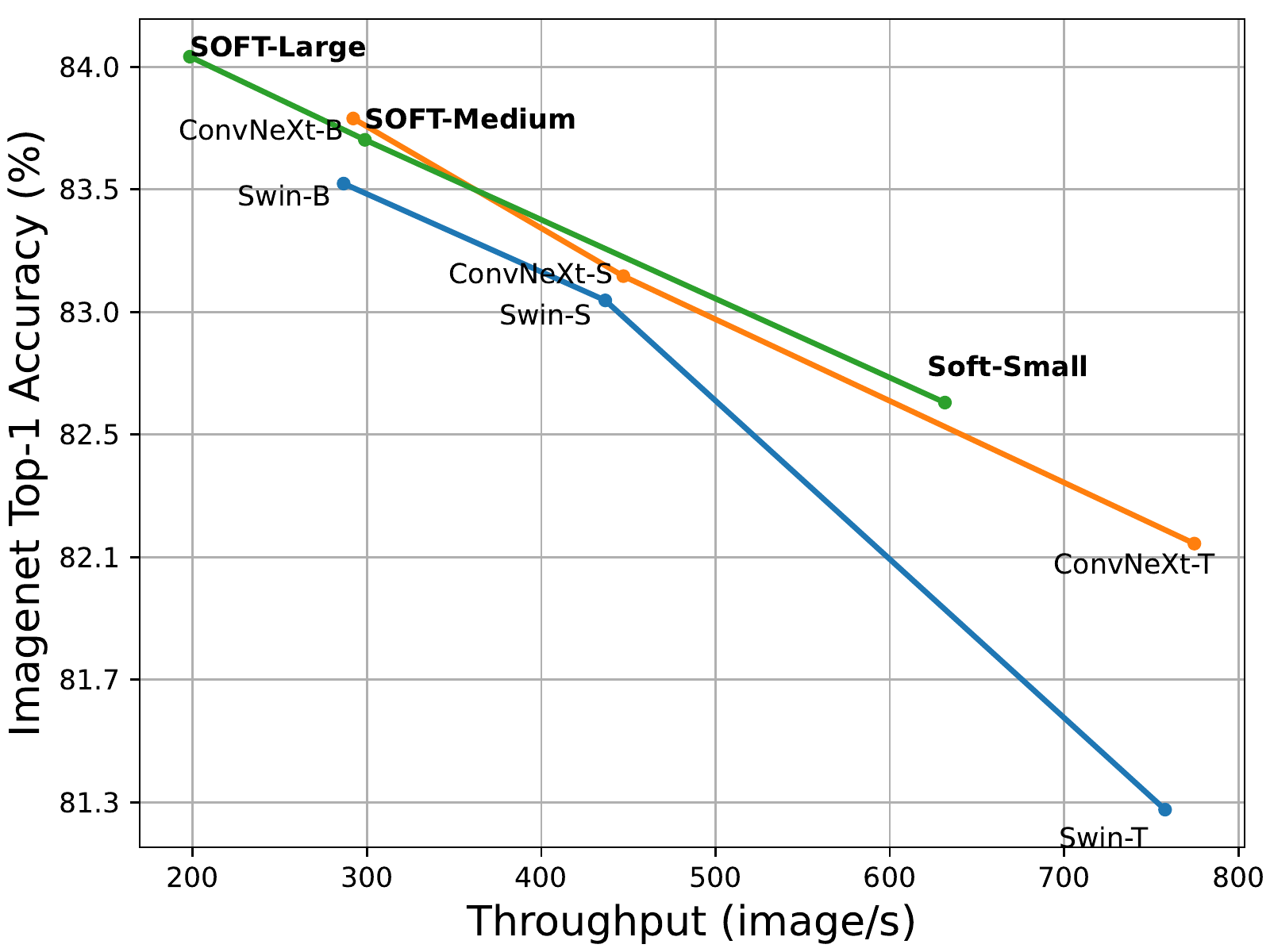}
\caption{\major{A comparison of Top-1 classification accuracy on the ImageNet validation set~\citep{deng2009imagenet} with respect to inference throughput for various models. Our comparison includes CNN models such as ConvNext~\citep{liu2022convnet}, as well as Transformer models like Swin~\citep{liu2021swin}. In this comparison, models positioned closer to the top-right indicate superior performance, balancing both accuracy and throughput effectively. 
Inference throughput is measured on a V100 GPU, following \citep{liu2021swin, liu2022convnet}.}
}
\vspace{-1em}
\label{fig:fps_top1}
\end{figure}

%% file: table/retina_od.tex
\begin{table*}[htb]
\caption{\major{Object detection (RetinaNet~\citep{lin2017focal})  results  on COCO~\citep{lin2014microsoft} \texttt{val2007}.
We report the results by training 12 epoch ($1\times$) schedule. 
\texttt{\#P} stands for parameter size. 
\texttt{AP}$^b$ represents bounding box AP.
Note, SOFT cannot converge in training.
}
}
\renewcommand{\arraystretch}{1.5} 
\begin{tabular*}{\textwidth}{@{\extracolsep\fill}lccccccc}
\hline

\multirow{2}{*}{Backbone}  & \multicolumn{7}{c}{RetinaNet 1$\times$} \\
\cmidrule{2-8}
& \#P  & AP & AP$_{50}$ & AP$_{75}$ & AP$_S$ & AP$_M$ & AP$_L$ \\ 

\hline

\hline

ResNet18~\citep{he2016deep} &21M & 31.8 & 49.6 & 33.6 & 16.3 & 34.3 & 43.2  \\
PVT-Tiny~\citep{wang2021pyramid} &23M& {36.7}& {56.9}& {38.9}& {22.6}& {38.8} &{50.0}   \\
PVTv2-B1 ~\citep{wang2022pvt} & 23M & 41.2 & 61.9 & 43.9 & 25.4 & 44.5 & 54.3 \\
\rerevision{\textbf{SOFT++-Tiny}} & \rerevision{23M} & \rerevision{\textbf{41.9}} & \rerevision{\textbf{62.7}} & \rerevision{\textbf{44.7}} & \rerevision{\textbf{27.8}} & \rerevision{\textbf{45.4}} & \rerevision{\textbf{55.6}} \\

ResNet50~\citep{he2016deep} &38M & 36.3 & 55.3 & 38.6 & 19.3 & 40.0 & 48.8 \\
PVT-Small~\citep{wang2021pyramid} & {34M} & {40.4} & {61.3} & {43.0} & {25.0} & {42.9} & {55.7} \\
ViL-S ~\citep{zhang2021multi} & 35M & 41.6 & 62.5 & 44.1 & 24.9 & 44.6 & 56.2 \\
Swin-T\citep{liu2021swin} & 38M & 41.7 & 61.2 & 43.2 & 26.0 & 44.3 & 54.5 \\
\rerevision{\textbf{SOFT++-Small}} & \rerevision{38M} & \rerevision{\textbf{43.7}} & \rerevision{\textbf{64.9}} & \rerevision{\textbf{46.8}} & \rerevision{\textbf{28.7}} & \rerevision{\textbf{47.4}} & \rerevision{\textbf{57.6}} \\

ResNet101~\citep{he2016deep} &57M  & 38.5 & 57.8 & 41.2 & 21.4 & 42.6 & 51.1 \\
PVT-Medium~\citep{wang2021pyramid} &54M & {41.9} & {63.1} & {44.3} & {25.0} & {44.9} & {57.6} \\
ViL-M ~\citep{zhang2021multi} & 50M & 42.9 & 64.0 & 45.4 & 27.0 & 46.1 & 57.2 \\
Swin-S\citep{liu2021swin} & 60M & 43.0 & 63.8 & 45.7 & 27.1 & 46.9 & 57.2 \\
\rerevision{\textbf{SOFT++-Medium}} & \rerevision{59M} & \rerevision{\textbf{44.3}} & \rerevision{\textbf{64.7}} & \rerevision{\textbf{47.4}} & \rerevision{\textbf{29.0}} & \rerevision{\textbf{48.2}} & \rerevision{\textbf{59.9}} \\

ResNeXt101 ~\citep{xie2017aggregated} & 95M & 41.0 & 60.9 & 44.0 & 23.9 & 45.2 & 54.0 \\
PVT-Large~\citep{wang2021pyramid} & 71M & {42.6} & {63.7} & {45.4} & {25.8} & {46.0} & {58.4}  \\
Swin-B ~\citep{liu2021swin} & 98M & 44.7 & 65.9 & 49.2 & - & - & - \\
\rerevision{\textbf{SOFT++-Large}} & \rerevision{98M} & \rerevision{\textbf{47.0}} & \rerevision{\textbf{67.8}} & \rerevision{\textbf{50.4}} & \rerevision{\textbf{30.2}} & \rerevision{\textbf{50.9}} & \rerevision{\textbf{62.0}} \\

\hline

\end{tabular*}

\label{tab:retina}
\end{table*}

%% file: table/maskrcnn_inst.tex
\begin{table*}[htb]
\caption{\major{Instance segmentation (Mask R-CNN~\citep{he2017mask}) results  on COCO~\citep{lin2014microsoft} \texttt{val2007}.
We report the results by training 12 epoch ($1\times$) schedule. 
\texttt{\#P} stands for parameter size. 
\texttt{AP}$^m$ represents mask AP respectively.
Note, SOFT cannot converge in training.}
}

\renewcommand{\arraystretch}{1.5} 
\begin{tabular*}{\textwidth}{@{\extracolsep\fill}lccccccc}

\hline
\multirow{2}{*}{Backbone}  & \multicolumn{7}{c}{Mask R-CNN 1$\times$} \\
\cmidrule{2-8}
& \#P  & AP$^b$ & AP$^b_{50}$ & AP$^b_{75}$ & AP$^m$ & AP$^m_{50}$ & AP$^m_{75}$ \\ 

\hline

\hline

ResNet18~\citep{he2016deep} &31M & 34.0 & 54.0 & 36.7 & 31.2 & 51.0 & 32.7\\
PVT-Tiny~\citep{wang2021pyramid} &33M & {36.7} & {59.2} & {39.3} & {35.1} & {56.7} & {37.3} \\
PVTv2-B1 ~\citep{wang2022pvt} & 33M & 41.8 & 64.3 & 45.9 & 38.8 & 61.2 & 41.6 \\
\rerevision{\textbf{SOFT++-Tiny}}  &\rerevision{32M} & \rerevision{\textbf{41.2}} & \rerevision{\textbf{63.7}} & \rerevision{\textbf{44.7}} & \rerevision{\textbf{38.2}} & \rerevision{\textbf{61.0}} & \rerevision{\textbf{41.0}}\\

ResNet50~\citep{he2016deep} & 44M & 38.0 & 58.6 & 41.4 & 34.4 & 55.1 & 36.7\\
PVT-Small~\citep{wang2021pyramid} &{44M} &{40.4} & {62.9} & {43.8} & {37.8} & {60.1} & {40.3}\\
ViL-S ~\citep{zhang2021multi} & 45M & 41.8 & 64.1 & 45.1 & 38.5 & 61.1 & 41.4 \\
Swin-T\citep{liu2021swin} & 48M& 42.7 & 65.2 & 46.8 & 39.3 & 62.2 & 42.2 \\
\rerevision{\textbf{SOFT++-Small}} & \rerevision{48M} & \rerevision{\textbf{43.8}} & \rerevision{\textbf{66.0}} & \rerevision{\textbf{47.5}} & \rerevision{\textbf{40.1}} & \rerevision{\textbf{63.0}} & \rerevision{\textbf{43.0}}\\

ResNet101~\citep{he2016deep} &63M & 40.4 & 61.1 & 44.2 & 36.4 & 57.7 & 38.8 \\
PVT-Medium~\citep{wang2021pyramid} &64M & {42.0} &{64.4} &45.6 &{39.0}& {61.6}& {42.1}\\
ViL-M ~\citep{zhang2021multi} & 60M & 43.4 & 65.9 & 47.0 & 39.7 & 62.8 & 42.1 \\
Swin-S\citep{liu2021swin} & 69M& 45.6 & 67.4 & 50.0 & 41.2 & 64.5& 44.3\\
\rerevision{\textbf{SOFT++-Medium}}  & \rerevision{69M} & \rerevision{\textbf{46.6}} & \rerevision{\textbf{67.8}} & \rerevision{\textbf{51.2}} & \rerevision{\textbf{42.0}} & \rerevision{\textbf{64.8}} & \rerevision{\textbf{45.2}}\\

ResNeXt101 ~\citep{xie2017aggregated} & 101M & 42.8 & 63.8 & 47.3 & 38.4 & 60.6 & 41.3 \\
PVT-Large~\citep{wang2021pyramid} & 81M& {42.9}& {65.0} & 46.6 &{39.5}& {61.9}& {42.5}  \\
Swin-B ~\citep{liu2021swin} & 107M & 45.5 & - & - & 41.3 & - & -\\
\rerevision{\textbf{SOFT++-Large}} & \rerevision{106M} & \rerevision{\textbf{47.0}} & \rerevision{\textbf{68.3}} & \rerevision{\textbf{51.7}} & \rerevision{\textbf{42.2}} & \rerevision{\textbf{65.2}} & \rerevision{\textbf{45.4}}\\

\hline
\end{tabular*}

\label{tab:maskrcnn}
\end{table*}

%% file: figure/attention_heatmap.tex
\begin{figure*}[!htb]\centering
\includegraphics[width=1.0\linewidth]{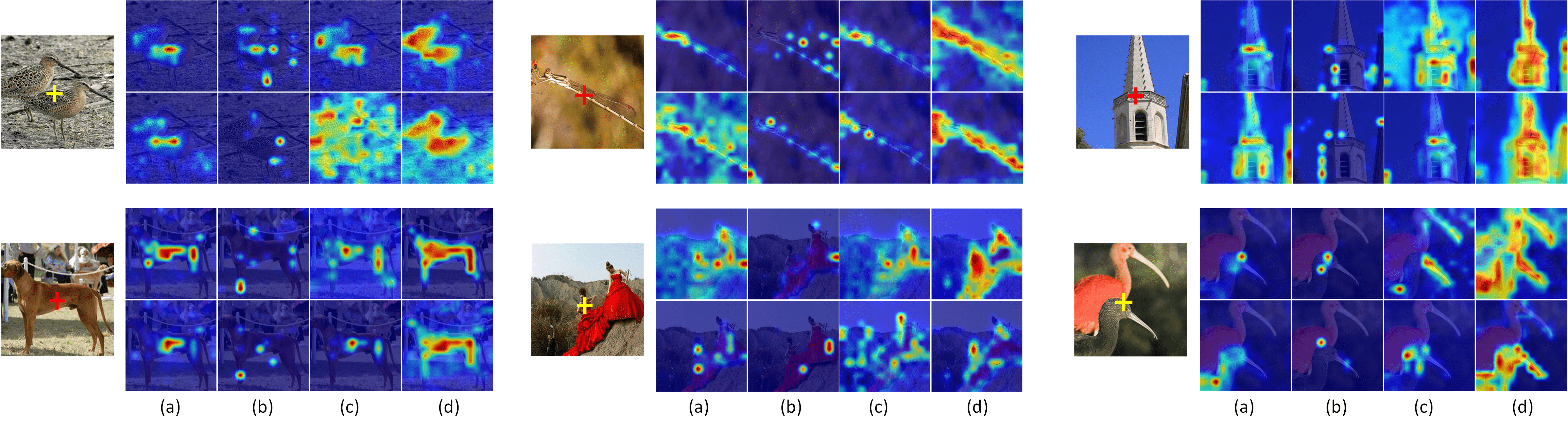}
\caption{\major{Comparison of attention heatmaps for a selected query patch (indicated by a cross "+") against all patches in an image. 
Heatmaps are derived from the first head's corresponding row in the attention maps, as calculated by Equation~\ref{eq:reg_norm_attn}. These heatmaps are normalized to a 0-1 scale, with warmer colors indicating higher relevance. The model variants compared are: \textbf{(a)} Transformer~\citep{vaswani2017attention}, \textbf{(b)} Performer~\citep{choromanski2020rethinking}, \textbf{(c)} Nystromformer~\citep{xiong2021nystr}, and \textbf{(d)} Our SOFT approach. For additional examples, refer to Appendix~\ref{sec:attn_vis}.}}

\vspace{-1em}
\label{fig:attention_heat}
\end{figure*}

%% file: table/seqlength_classification.tex
\begin{table*}[ht]
\renewcommand{\arraystretch}{1.5} 
\caption{
\rerevision{The evaluation results on the validation set of ImageNet-1K~\citep{deng2009imagenet} at various input sizes. The models were trained on the size of $224\times224$.
The down-sampling ratio of the patch embedding remains at 4. 
}}
\centering
\begin{tabular*}{\textwidth}{@{\extracolsep\fill}lccc}

    \hline
    Test input size & $224\times224$ & $384\times384$ & $512\times512$ \\
    \hline

    \hline
    Token sequence length & 3136& 9216 & 16384 \\

    \rerevision{SOFT} & \rerevision{82.5} & \rerevision{33.7} & \rerevision{0.14}\\
    \rerevision{SOFT++} & \rerevision{82.6} & \rerevision{82.6} & \rerevision{80.2}\\

    \hline
\end{tabular*}
\label{tab:seqlen_cls}
\end{table*}

%% file: table/soft_segmentation.tex
\begin{table*}[t]
\caption{\rerevision{Semantic segmentation performance using our HLG Transformer with UperNet~\citep{xiao2018unified} and SETR-PUP on Cityscapes validation set. 
Single-scale inference and 40k training schedules are used.
Note, SOFT cannot converge in training.
}
}
\label{tab:soft_segmentation}
\renewcommand{\arraystretch}{1.5} 
\begin{tabular*}{\textwidth}{@{\extracolsep\fill}llccccc}

\hline
\multirow{2}{*}{Method} & \multirow{2}{*}{Backbone} & \multirow{2}{*}{\#P} & \multicolumn{2}{c}{Cityscapes} & \multicolumn{2}{c}{ADE20K}\\
  &  &  & FLOPs & mIoU & FLOPs & mIoU\\
\hline

\hline
 FCN~\citep{long2015fully} & ResNet-101 & 68M  & 619G  & 76.6  & 276G & 39.9\\
 PSPNet~\citep{zhao2017pyramid} & ResNet-101 & 68M  & 576G  & 78.5 & 256G & 44.4 \\
 DLabV3+~\citep{chen2018encoder} & ResNet-101 & 62M & 571G & 79.3 & 262G & 46.9 \\
 CCNet~\citep{huang2019ccnet} & ResNet-101 & 68M & 625G & 80.2 & 278G & 43.7\\
 SETR~\citep{zheng2021rethinking} & ViT-Large & 318M & 1340G & 79.2 & 602G & 48.5\\
 OCRNet~\citep{yuan2020object} & HRNet-W48 & 70M & 972G & 81.1 & 164G & 43.2 \\
 UperNet~\citep{xiao2018unified} & Swin-T & 60M & - & - & 236G & 44.5\\
 UperNet~\citep{xiao2018unified} & Swin-S & 81M & - & - & 259G & 47.6\\
 UperNet~\citep{xiao2018unified} & Swin-B & 121M & - & - & 297G & 48.1\\
 UperNet~\citep{xiao2018unified} & ConvNext & 60M & - & - & 235G & 46.0\\
 UperNet~\citep{xiao2018unified} & ConvNext & 82M & - & - & 256G & 48.7\\
 UperNet~\citep{xiao2018unified} & ConvNext & 122M & - & - & 293G & 49.1\\

 \rerevision{UperNet}~\citep{xiao2018unified} & \textbf{\rerevision{SOFT++-Small}} & \rerevision{60M} & \rerevision{545G} & \rerevision{81.2} & \rerevision{237G} & \rerevision{46.5} \\
 \rerevision{UperNet}~\citep{xiao2018unified} & \textbf{\rerevision{SOFT++-Med}}  & \rerevision{81M} & \rerevision{607G} & \rerevision{82.0} & \rerevision{260G} & \rerevision{48.9} \\
 \rerevision{UperNet}~\citep{xiao2018unified}  &\textbf{\rerevision{SOFT++-Large}}  & \rerevision{121M} & \rerevision{899G} & \rerevision{\textbf{82.6}} & \rerevision{301G} & \rerevision{\textbf{49.2}} \\

\hline

\end{tabular*}
\end{table*}

%% file: table/symmetric.tex
\begin{table*}[ht]
\caption{\major{Ablations on assessing the Impact of Symmetric Kernels on ImageNet-1K~\citep{deng2009imagenet} image classification. As baselines for our study, we employ two of the most common backbones: ViT-small/16~\citep{dosovitskiy2020image} and Swin-T~\citep{liu2021swin}.
}}

\centering
\renewcommand{\arraystretch}{1.5} 
\begin{tabular*}{\textwidth}{@{\extracolsep\fill}lcc}

    \hline
    Model & Asymmetric ($Q \neq K$) & Symmetric ($Q = K$)\\
    \hline
    
    \hline
    ViT-Small/16~\citep{dosovitskiy2020image}& 80.8 & 80.5 \\
    Swin-T~\citep{liu2021swin} & 81.3 & 80.9\\

    \hline
\end{tabular*}
\label{tab:abl_sym}
\end{table*}

%% file: table/bottleneck.tex
\begin{table*}[ht]
\caption{Ablations on bottleneck token sequence length.}
\centering
\renewcommand{\arraystretch}{1.5} 
\begin{tabular*}{\textwidth}{@{\extracolsep\fill}cccc}

    \hline
    Bottleneck & Memory & FLOPs & Top-1 \%\\
    \hline
    
    \hline
    36 & 15.1GB & 1.9G & 80.5 \\
    49 & 15.8GB & 1.9G & 80.9 \\
    64 & 16.9GB & 2.0G & 80.9 \\
    81 & 18.5GB & 2.1G & 80.4 \\

    \hline
\end{tabular*}
\label{tab:bottleneck token}
\end{table*}

%% file: table/sampling.tex
\begin{table*}[ht]
\caption{Ablations on sampling methods.}
\centering
\renewcommand{\arraystretch}{1.5} 
\begin{tabular*}{\textwidth}{@{\extracolsep\fill}lccc}

    \hline
    Sampling methods & Params & FLOPs & Top-1 \%\\
    \hline
    
    \hline
    Convolution & 13.07M & 2.0G & 80.9 \\
    Random sampling & 12.96M & 1.9G & 80.8 \\
    Biased sampling & 12.96M & 1.9G & 79.9 \\
    Average pooling & 12.96M & 1.9G & 80.8 \\

    \hline
\end{tabular*}
\label{tab:samplingmethods}
\end{table*}

%% file: figure/newtoniter_rank2.tex
\begin{figure*}[htb]
    \centering
    \subfloat{
         \includegraphics[width=0.33\textwidth]{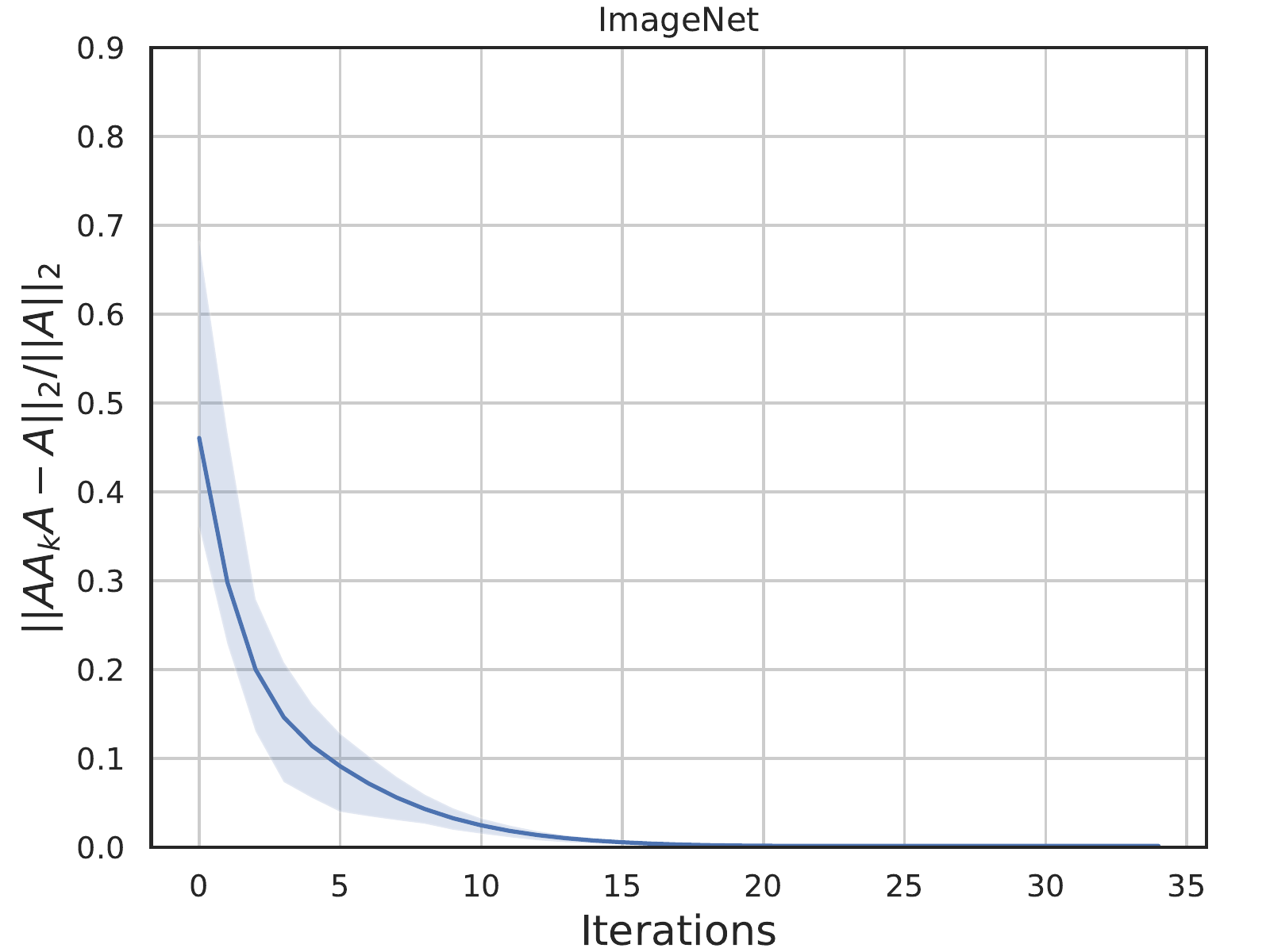}
    }
    \subfloat{
        \includegraphics[width=0.33\textwidth]{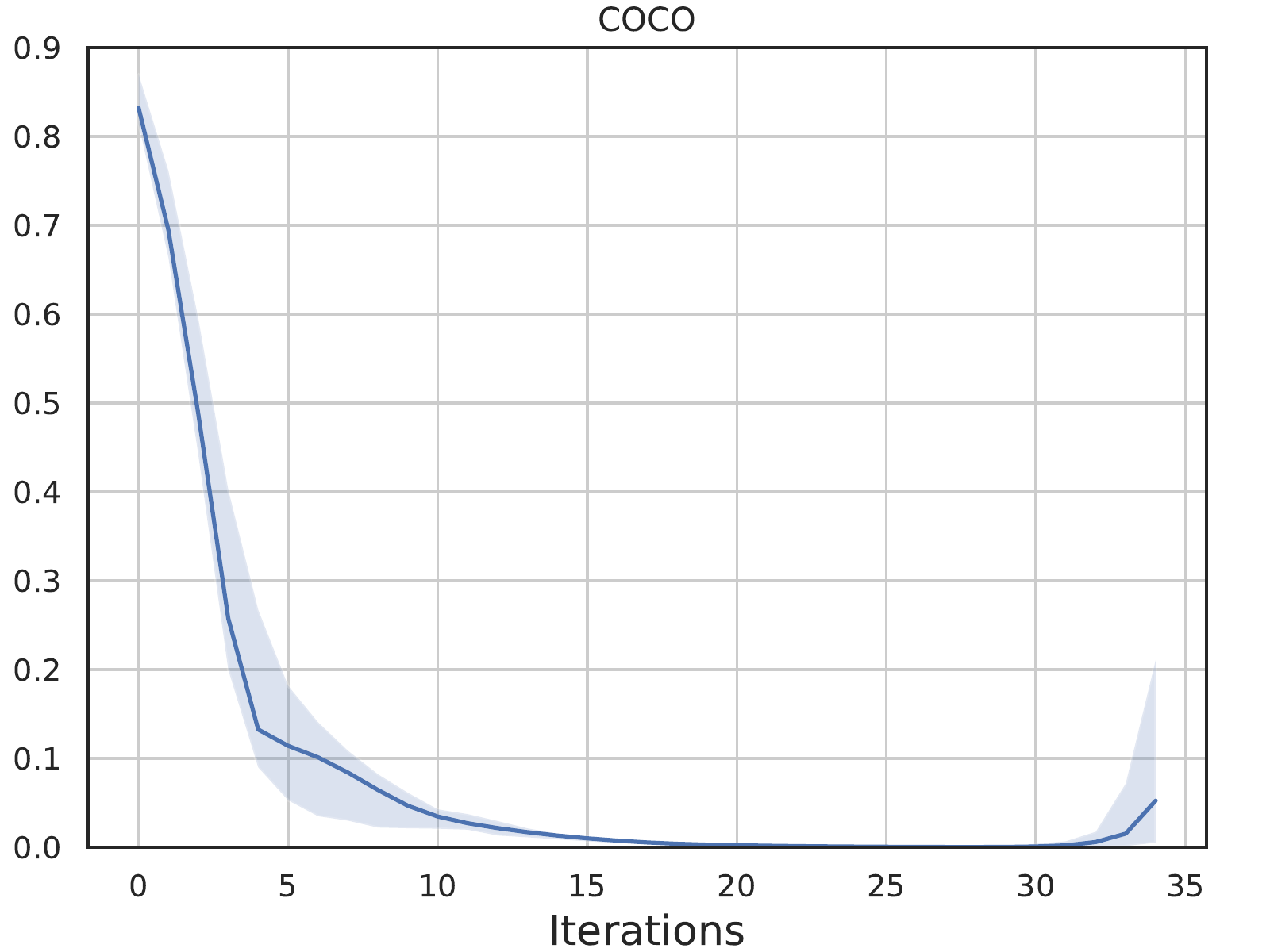}
    }
    \subfloat{
        \includegraphics[width=0.33\textwidth]{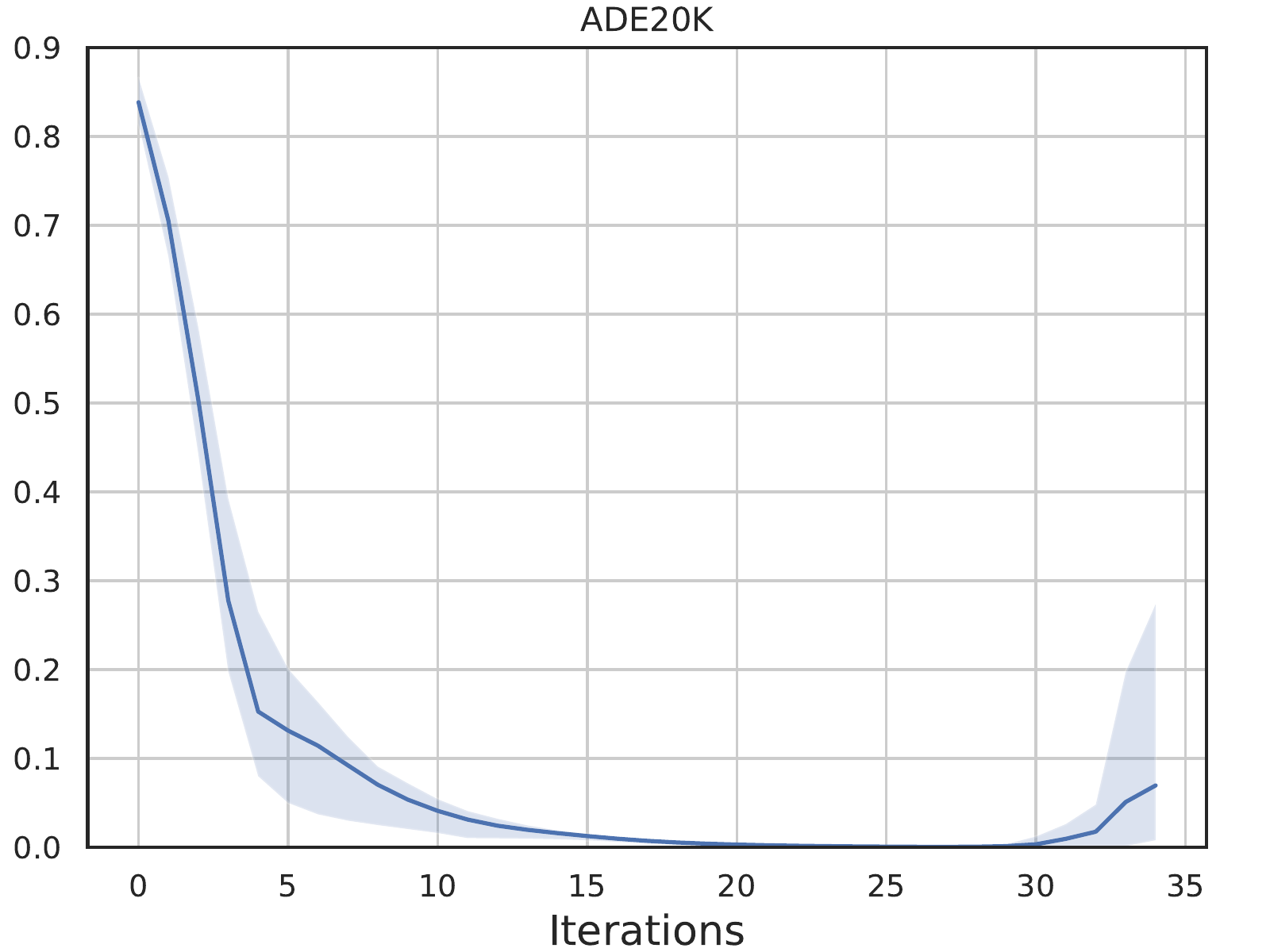}
    }
    \caption{\extension{Convergence analysis for the approximation of Moore-Penrose inverse on ImageNet, COCO and ADE20K separately.
    SOFT-Tiny is used.
    We measure $\|AA_kA-A\|_p/\|A\|_2$ for 100 input images on each dataset.
    The solid line shows the average convergence metric, while the shallow area indicates the upper bound and lower bound.
    }}
    \label{fig:newtoniterationrelerr}
\end{figure*}

%% file: table/pyramidal-overlapping.tex
\begin{table*}[!htb]
\centering
\caption{(a) Ablations on pyramidal architecture. 
    (b) Ablations on overlapped convolution.}
\renewcommand{\arraystretch}{1.5} 
        \begin{tabular*}{\textwidth}{@{\extracolsep\fill}lccclcc}
            \cmidrule{1-3} \cmidrule{5-7}
            \cmidrule{1-3} \cmidrule{5-7}
            \cmidrule{1-3} \cmidrule{5-7}
            Methods & \major{Pyramidal?}  & Top-1 \% & & Methods & \major{Overlapped?} & Top-1 \% \\
            \cmidrule{1-3} \cmidrule{5-7}
            DeiT-S~\citep{touvron2021training} & \XSolidBrush  & 79.8 & & PVT~\citep{wang2021pyramid} & \XSolidBrush & 75.1\\
            SOFT++ & \XSolidBrush & 80.1 & & SOFT++ & \XSolidBrush & 78.4\\
            SOFT++ & \checkmark & 82.4 & & SOFT++ & \checkmark & 80.9\\
            \cmidrule{1-3} \cmidrule{5-7}
            \cmidrule{1-3} \cmidrule{5-7}
            \cmidrule{1-3} \cmidrule{5-7}
        \end{tabular*}

    \label{tab:pyramidal-overlapping}
\end{table*}

%% file: table/lra.tex
\begin{table*}[!htb]
\setlength{\tabcolsep}{7pt} 
\renewcommand{\arraystretch}{1.2} 
\caption{Comparison of different linear/efficient Transformer variants on Long Range Arena~\citep{tay2020long}, based on its official configuration. Our SOFT surpasses previous efficient methods on three tasks.
}
\renewcommand{\arraystretch}{1.5} 
\begin{tabular*}{\textwidth}{@{\extracolsep\fill}lccccc}

\hline
Methods & Listops(2K) & Text(4K) & Retrieval(4K) & Image(1K) & Avg. \%\\
\hline

\hline

Transformer~\citep{vaswani2017attention} & 37.10 & 65.02 & 79.35 & 38.20 & 54.92\\
Reformer~\citep{kitaev2020reformer}  &  19.05 & 64.88 & 78.64 & 43.29 & 51.47\\
Linformer~\citep{wang2020linformer} & 37.25 & 55.91 & 79.37 & 37.84 & 52.59  \\
Performer~\citep{choromanski2020rethinking} & 18.80 & 63.81 & 78.62 & 37.07 & 49.58\\
Nystr{\"o}mformer~\citep{xiong2021nystr} & 37.15 & \textbf{65.52} & 79.56 & 41.58 & 55.95\\
\textbf{SOFT} & \textbf{37.40} & 63.49 & \textbf{81.77} & \textbf{46.91} & \textbf{57.39}\\

\hline
\end{tabular*}
\label{tab:lra}
\end{table*}

%% file: figure/symmetric_norm_eigen.tex
\begin{figure*}[htp]
    \centering
    \hspace{-1em}
    \subfloat[\label{fig:eigen}]{
         \includegraphics[width=0.485\linewidth]{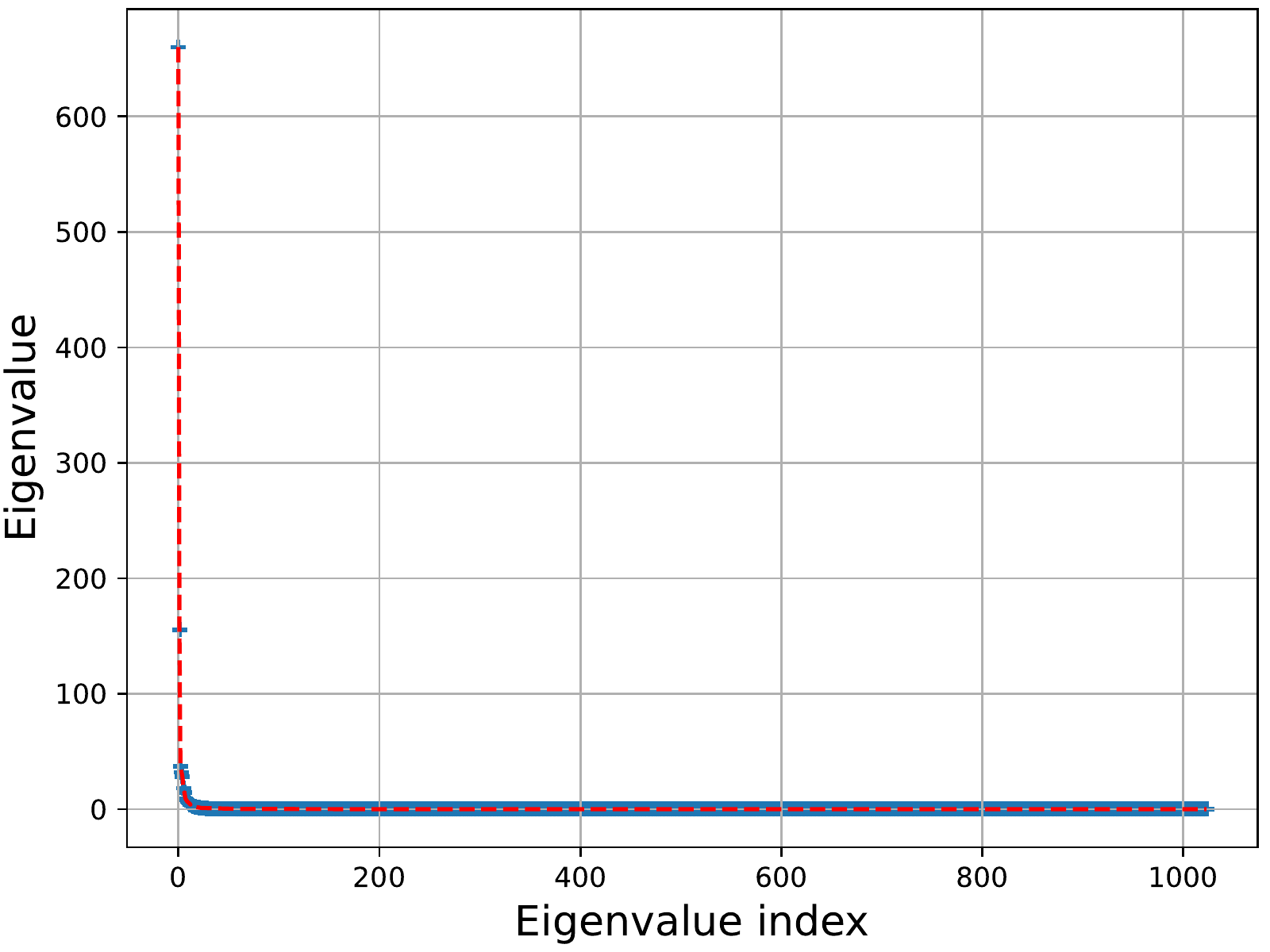}
    }
    \subfloat[\label{fig:eigen_vs_m}]{
        \includegraphics[width=0.5\linewidth]{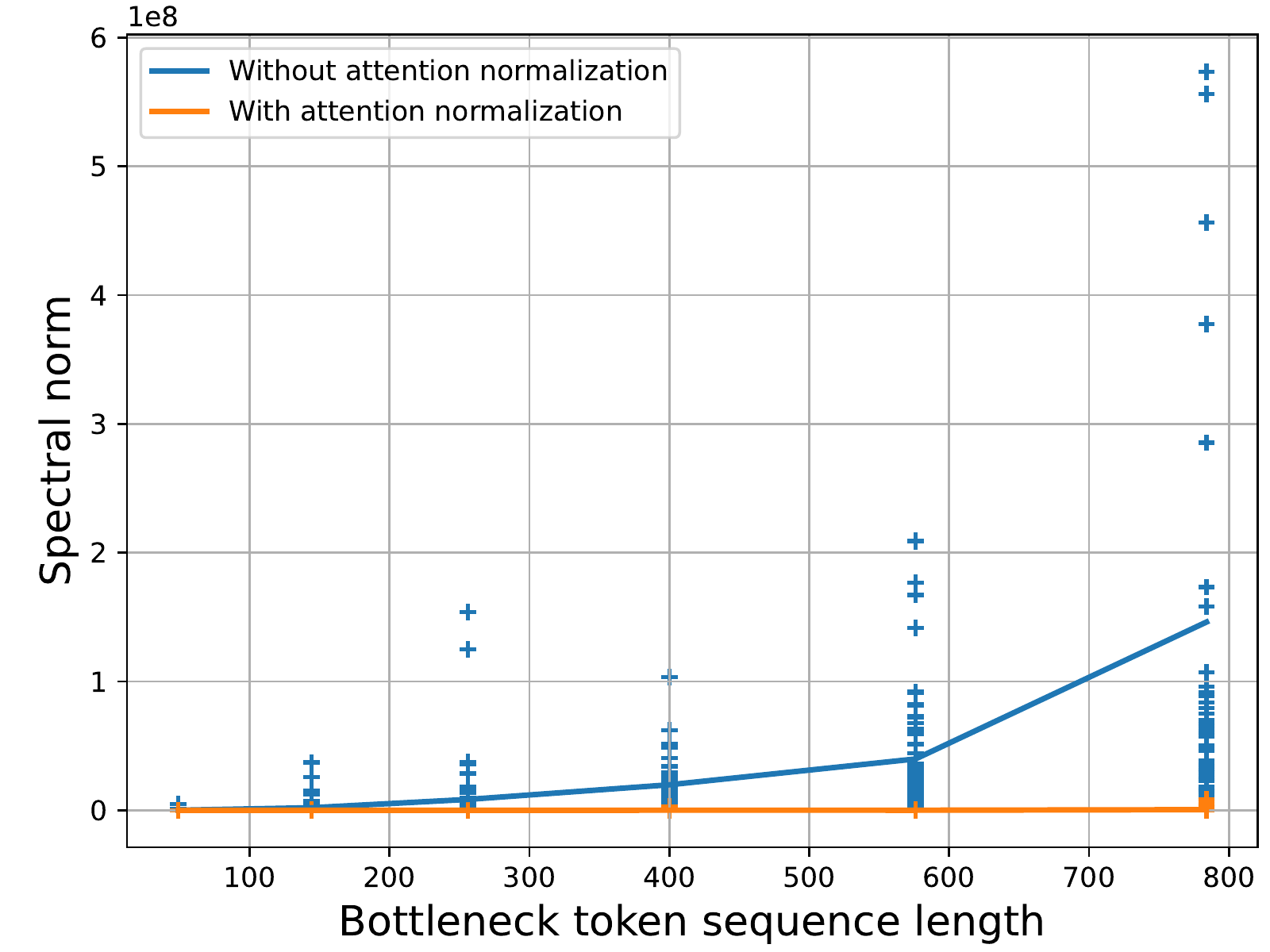}
    }
    \caption{\extension{(a) Eigenvalues of a specific bottleneck matrix $A\in \mathbb{R}^{1024\times 1024}$. 
    The auxiliary (red dash) line connects all the eigenvalues. (b) Comparing the spectral norm of self-attention matrix $A^\dag$ (without normalization, blue) and $D^{-1/2}A^\dag D^{-1/2}$ (with normalization, orange) under a variety of bottleneck token sequence lengths. ``+" represents individual samples; The lines connect the average norm at each token sequence length.}}
\end{figure*}

%% file: sections/5_conclusion.tex
\section{Conclusions}
\label{conclusions}
In this work, we have introduced a novel 
softmax-free self-attention (SOFT) mechanism
for linearizing Transformer's complexity
in space and time.
Unlike existing linear Transformers
that aim to approximate the conventional softmax based self-attention,
SOFT employs a Gaussian kernel based attention which
eliminates the need for softmax normalization.
This design enables a full self-attention matrix to be approximated via low-rank matrix decomposition.
The robustness of this proposed approximation is achieved by calculating its Moore-Penrose inverse using  a  Newton-Raphson method and \extension{an efficient symmetric attention normalization.} Extensive experiments show that SOFT yields superior
trade-off in accuracy and complexity on a variety of vision and language tasks.
\\
\\
{\small
\noindent\textbf{Acknowledgements}
This work was supported in part by STI2030-Major Projects (Grant No. 2021ZD0200204), National Natural Science Foundation of China (Grant No. 62106050 and 62376060),
Natural Science Foundation of Shanghai (Grant No. 22ZR1407500) and Lingang Laboratory (Grant No. LG-QS-202202-07).
\\
\\
\noindent\textbf{Data availability statement}
The datasets generated during and/or analysed during the current study are available in the Imagenet \citep{deng2009imagenet} (\url{https://www.image-net.org/}), COCO \citep{lin2014microsoft} (\url{https://cocodataset.org}), ADE20K \citep{zhou2019semantic} (\url{https://groups.csail.mit.edu/vision/datasets/ADE20K/}), Cityscapes \citep{cordts2016cityscapes} (\url{https://www.cityscapes-dataset.com}), Long Range Arena \citep{tay2020long} (\url{https://github.com/google-research/long-range-arena}) repositories.
}

%% file: sections/7-appendix_short.tex
\begin{appendix}

\section{Nystr{\"o}m method}
\label{sec:nystrom}
Nystr{\"o}m method~\cite{NIPS2000_nyst} aims to calculate a low-rank approximation for a Gram matrix. For Transformers, the self-attention matrix can be viewed as a Gram matrix $S$ with a Gaussian kernel $k$ applied to the query $Q$, with each element $ S_{ij}$ expressed as:
\begin{equation}
    S_{ij} =k\big(Q_{i,:},Q_{j,:}\big) = \text{exp} (-\frac{\|Q_{i,:}-Q_{j,:}\|_2^2}{2 \sqrt d}),
\end{equation}
$k(x,y)$ means operating Gaussian kernel $k$ to $(x,y)$, which can be written in the feature space as:
\begin{equation}
    k(x,y)=\sum\limits_{i=1}^n \lambda_i \phi_i(x)\phi_i(y),
\end{equation}
$n$ is the dimension of a feature space, $\lambda_i$ denotes the eigenvalue and $\phi_i$ denotes the eigenfunction of kernel $k$. According to the eigenfunction's definition, we can get:
\begin{equation}
    \int k(y,x)\phi_i(x)p(x)dx=\lambda_i \phi_i(y),
\end{equation}
where $p(x)$ is the probability distribution of $x$.
And \{$\phi_i(x)$\} are $p$-orthogonal:
\begin{equation}
    \int \phi_i(x)\phi_j(x)p(x)dx=\delta_{ij}.
\end{equation}
$\delta_{ij}$ is $0$ when $i \neq j$, $1$ when $i = j$. To get an approximation of the eigenfunctions, we sample $\{x_1,x_2,\cdots,x_q\}$ from $p(x)$, then:
\begin{equation}
    \frac{1}{q}\sum\limits_{t=1}^q k(y,x_t)\phi_i(x_t) \approx \lambda_i \phi_i(y),
\end{equation}
\begin{equation}
    \frac{1}{q}\sum\limits_{t=1}^q \phi_i(x_t)\phi_j(x_t) \approx \delta_{ij}.
\end{equation}

This inspires us to approximate the Gram matrix $S$. Let $S^{(m)}$ be a submatrix of $S$, consisting of $m \times m$ elements from $S$. Gram matrix is a symmetric positive semi-definite matrix, so it has a spectral decomposition:
\begin{equation}
    S^{(m)}U^{(m)}=U^{(m)} \Lambda^{(m)},
\end{equation}
where $U^{(m)}$ is column orthogonal and $\Lambda^{(m)}$ is a diagonal matrix with the diagonal elements as the eigenvalues of $S^{(m)}$. Substituting the $y$ to $x_j$ and applying the approximation above to $S$, we can get:
\begin{equation}
    \phi_i(x_j) \approx \sqrt m U_{j,i}^{(m)}, \quad \lambda_i \approx \frac{\lambda_i^{(m)}}{m},
\end{equation}
\begin{equation}
    \phi_i(y) \approx \frac{\sqrt{m}}{\lambda_i^{(m)}}\sum\limits_{t=1}^{m}k(y,x_t)\phi_i(x_t),
\end{equation}
$\lambda_i$ is eigenvalue of $S$ and $\lambda_i^{(m)}$ is the eigenvalue of $S^{(m)}$.
Denote $\Tilde{S}$ as the rank-$m$ approximation of $S$ and $\Tilde{U}, \Tilde{\Lambda}$ as the approximation for spectral decomposition of $S$.  Now we can get an approximation of $S$ with rank $m$:
\begin{equation}
    \Tilde{S}=\Tilde{U}\Tilde{\Lambda}\Tilde{U}^T=\sum\limits_{t=1}^m \Tilde{\lambda_t}^{(n)} \Tilde{u}_t^{(n)}(\Tilde{u}_t^{(n)})^T.
\end{equation}

Similarly, we have:
\begin{equation}
    \phi_i(x_j) \approx \sqrt{n} U_{j,i}(n),\quad \lambda_i \approx \frac{\Tilde{\lambda_i}^{(n)}}{n}.
\end{equation}

Thus
\begin{equation}
    \Tilde{\lambda_i}^{(n)} \approx \frac{n \lambda_i^{(m)}}{m},
\end{equation}
\begin{equation}
    \Tilde{u}_t^{(n)} \approx \sqrt{\frac{m}{n}}\frac{1}{\lambda_t^{(m)}}S_{n,m} u_t^{(m)}.
\end{equation}

Then we get an approximation of $S$: $\Tilde{S} \approx S_{n,m} S_{m,m}^{\dagger} S_{m,n}$. 
$S$ has a block representation below:
\begin{equation}
    S=
    \begin{bmatrix}
    S_{m,m} & S_{m,n-m} \\
    S_{n-m,m} & S_{n-m,n-m}
    \end{bmatrix}.
\end{equation}

\section{Newton method}
\label{suppsec:newton}
\noindent {\bf Proof of Proposition~\ref{theo}} 
When $\alpha$ is sufficiently small, $A_{k+1}=2A_k-A_k A A_k$,  $A_k $ converges to $A^{\dagger}$.
\begin{proof}
$A$ is a symmetric positive semi-definite matrix and $A_{ij}\leq 1$, $\forall 1\leq i,j \leq n$, $A_{ii}=1,\quad 1\leq i\leq n$ in our case. $A_0$ is chosen to be $\alpha A$, so the $A_k$ can be written as $A_k=C_k A =AD_k$ for some matrix $C_k,D_k$, leading to the fact that 
\begin{equation}
     A^{\dagger}AA_k=A_k,\quad A_k A A^{\dagger}=A_k.
\end{equation}
This is because $A_{k+1}=A_k(2I_n - AA_k)=(2I_n - A_kA)A_k$ and $A_0=\alpha A$.
We make a difference between $A^{\dagger}$ and $A_{k+1}$:     
\begin{align}
     A^{\dagger}-A_{k+1} &= A^{\dagger} - 2A_k + A_k A A_k   \notag \\
    &= A^{\dagger} -A_k A A^{\dagger} -A^{\dagger} A A_k +A_k A A_k   \notag \\
    &= (A^{\dagger}-A_k) A (A^{\dagger}-A_k).
    \label{eq: iter}
\end{align}
We norm both sides of the equation above: 
\begin{align}
    \|  A^{\dagger}-A_{k+1} \|  &= \| ( A^{\dagger}-A_{k})A( A^{\dagger}-A_{k})\|  \notag \\
    &\leq \|  A^{\dagger}-A_{k} \|\|A( A^{\dagger}-A_{k})\|.
    \label{converge}
\end{align}
And we left multiply $A$ on the both sides of (Eq~\ref{eq: iter}), then norm the equation:
\begin{align}
    \|A A^{\dagger}-A A_{k+1}\| &=\|A (A^{\dagger}-A_k) A (A^{\dagger}-A_k)\| \notag \\
    & \leq \|A A^{\dagger} -A A_{k}\|^2 .
\end{align}
We choose $\alpha$ sufficiently small so that the initial value satisfy $\|AA^{\dagger}-A A_0\| < 1$. We set $\alpha =\frac{2}{\|A\|_1^2}$ to ensure it is small enough~\cite{ben1966iterative}. Then the $\|A A^\dagger - A A_{k}\|\rightarrow 0$, when $k \rightarrow \infty$. The inequality (\ref{converge}) implies that $A_k \rightarrow A^{\dagger},\ k \rightarrow \infty$.
\end{proof}

\noindent {\bf Proof of Equation~\ref{equ:inverse_back}} 
\begin{equation*}
    \nabla_x \mathcal{L} = -Y^\top (\nabla_Y \mathcal{L}) Y^\top, 
\end{equation*}
\begin{proof}
Since $XY=I$, we take derivatives on the both sides
\begin{align}
    \notag YdX + XdY &= 0\\
    dY &= -YdXY
\end{align}
For the loss function $\mathcal{L}$, we have
\begin{equation*}
    d\mathcal{L} = \langle \nabla_X \mathcal{L}, dX \rangle = \langle \nabla_Y \mathcal{L}, dY \rangle
\end{equation*}
Recall that the inner product $\langle A,B\rangle = \text{Tr}(A^\top B)$ and $\text{Tr}(ABC)=\text{Tr}(CAB)=\text{Tr}(BCA)$, we have
\begin{align*}
    \langle \nabla_X \mathcal{L}, dX \rangle &= \langle \nabla_Y \mathcal{L}, dY \rangle \\
    &= \langle \nabla_Y \mathcal{L}, -YdXY \rangle\\
    &= -\text{Tr}(\nabla_Y \mathcal{L}^\top YdXY )\\
    &= \langle -Y^\top \nabla_Y \mathcal{L} Y^\top, dX \rangle
\end{align*}
Therefore, Eq~\eqref{equ:inverse_back} is proved.
\end{proof}

\section{Attention normalization}
\label{suppsec:attn_norm}
\noindent{\bf Proof of Porpostion \ref{theo:inverse_eigen}}
Assume the bottleneck matrix of softmax-free attention , $A\in\mathbb{R}^{m\times m}$ is $k$-connected. If $\lambda_1\geq \lambda_2\geq \cdots \lambda_m \geq 0$ are eigenvalues of $A^\dag$, then $\lambda_1 = \mathcal{O}(m^2)$ and $\|A^\dag\|_2 = \mathcal{O}(m^2)$.

To prove Proposition~\ref{theo:inverse_eigen}, we first introduce the following lemmas.
\begin{lemma}
If $F$ is a non-singular and symmetric matrix, then $\|F\|_2\leq \|F\|_1$.
\end{lemma}
\begin{proof}
According to the definition of 2-Norm, we have $F^TFz=\mu^2z$ with $\mu = \|F\|_2$. Therefore, $\mu^2\|z\|_1 = \|F^TFz\|_2\leq$\\
$\|F^T\|_1\|F\|_1\|z\|_1 = \|F\|_{\infty}\|\|F_1\|z\|_1$.
Since $A$ is non-singular and symmetric, 
\begin{equation}
    \|F\|_2 \leq \sqrt{\|F\|_{\infty}\|F\|_1} = \|F\|_1
\end{equation}
\end{proof}
\begin{lemma}
\label{lemma:inv_pert}
If $F$ is a real matrix and $\|F\|_p < 1$, then
\begin{equation}
\label{equ:inv_pert}
    (I-F)^{-1} = \sum_{k=0}^\infty F^k
\end{equation}
with
\begin{equation}
    \|(I-F)^{-1}\|_p \leq \frac{1}{1-\|F\|_p}
\end{equation}
\end{lemma}
\begin{proof}
\begin{equation*}
    \left( \sum_{k=0}^N F^k \right)(I-F) = I- F^{N+1}
\end{equation*}
Since $\|F\|_p<1$, $\lim_{k\rightarrow \infty} F^k=0$. Thus,
\begin{equation*}
\left( \lim_{N\rightarrow \infty} \sum_{k=0}^N F^k \right)(I-F) = I
\end{equation*}
Multiply the equation by $(I-F)^{-1}$ we obtain Eq~\eqref{equ:inv_pert}. From that, we can easily get
\begin{equation*}
    \|(I-F)^{-1}\|_p \leq \sum_{k=0}^\infty \|F\|_p^k = \frac{1}{1-\|F\|_p}
\end{equation*}
\end{proof}
Now we begin to prove Proposition~\ref{theo:inverse_eigen}.
\begin{proof}
Since the bottleneck matrix softmax-free attention $\|A\|_1$ is not less than 1, we normalize the matrix as $A_n = D^{-1}A$, where $D = \text{diag}(A\mathbbm{1}_m)$.
From Lemma~\ref{lemma:inv_pert}, it follows that
\begin{equation*}
    \|A_n^{-1}\|_1 \leq \frac{1}{1-\|1-A_n\|_1}
\end{equation*}
Since we assume the bottleneck matrix is k-connected, and $k<<m$, 
\begin{equation*}
    \|I-A_n\|_1 \leq \frac{(m-k)}{m}
\end{equation*}
Then,
\begin{equation*}
    \|A_n^{-1}\|_1 \leq \frac{1}{1 - \frac{(m-k)}{m}} = \frac{m}{k}= \mathcal{O}(m)
\end{equation*}
Note that we $D=\mathcal{O}(m)$ by assumption that the bottleneck matrix is k-connected, it follows that
\begin{equation*}
    \|A^{-1}\|_1 = \|A_n^{-1}D\|_1 \leq \|A_n^{-1}\|_1\|D\|_1 = \mathcal{O}(m^2)
\end{equation*}
\end{proof}

\section{Non-linearized Gaussian kernel attention}
Instead of directly calculating the Gaussian kernel weights, in our formulation they are approximated. More specifically, the relation between any two tokens is reconstructed via sampled bottleneck tokens. As the number $m$ (\eg, 49) of bottleneck tokens is much smaller than the entire token sequence length, our attention matrix is of low-rank. This brings about two favorable consequences: \textbf{(I) }The model now focuses the attentive learning on latent salient information captured by the $m$ bottleneck tokens. \textbf{ (II) }The model becomes more robust against the underlying token noise due to the auto-encoder style reconstruction~\cite{ham}.
\extension{
This explains why a model with an approximated gram matrix performs better than one with a directly computed matrix. Further, we find that exact Gaussian kernel attention leads to training difficulties. 
As Proposition~4 reveals, this is due to lacking normalization that leads to explosion of the spectral norm
especially in long token sequence cases.
Big spectral norm could jeopardize the training and
tend to collapse the model.
}

\section{Attention visualization}
\label{sec:attn_vis}
Figure~\ref{fig:supattention} shows more visualization of the attention masks by various Transformers~\cite{vaswani2017attention,choromanski2020rethinking,xiong2021nystr} and our SOFT.
For each model, we show the output from the first two attention heads (up and down row).
It is noteworthy that SOFT exhibits better semantic diversity of the multi-head mechanism than other methods. 
Moreover, when we sample the patch at the boundary of multiple objects, SOFT is able to more precisely 
capture all these objects.
\input{figure/supple_attn_vis}
\end{appendix}

%% file: figure/supple_attn_vis.tex
\begin{figure*}[!htb]\centering
\includegraphics[width=\linewidth]{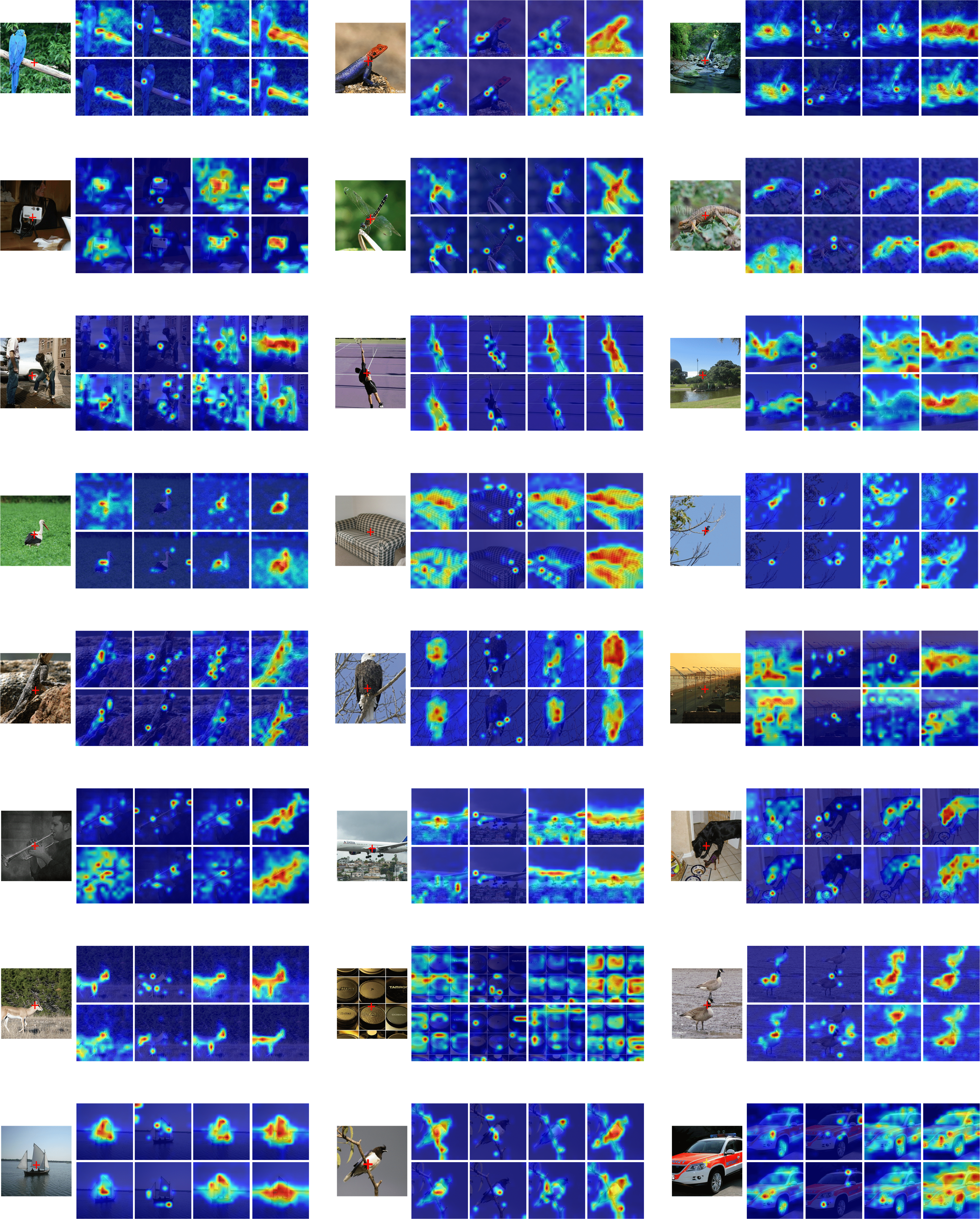}
\caption{
\major{Comparison of attention heatmaps for a selected query patch (indicated by a cross "+") against all patches in an image. 
Heatmaps are derived from the first head's corresponding row in the attention maps, as calculated by Equation~\ref{eq:reg_norm_attn}. These heatmaps are normalized to a 0-1 scale, with warmer colors indicating higher relevance. The model variants compared are: \textbf{(a)} Transformer~\citep{vaswani2017attention}, \textbf{(b)} Performer~\citep{choromanski2020rethinking}, \textbf{(c)} Nystromformer~\citep{xiong2021nystr}, and \textbf{(d)} Our SOFT approach.}}
\label{fig:supattention}
\end{figure*}